\documentclass{article}

\usepackage[preprint]{neurips_2026}

\usepackage[utf8]{inputenc}
\usepackage[T1]{fontenc}
\usepackage{booktabs}
\usepackage{amsfonts}
\usepackage{nicefrac}
\usepackage{microtype}
\usepackage{xcolor}
\usepackage{amsmath,amssymb,amsthm,mathtools}
\usepackage{graphicx}
\usepackage{subcaption}
\definecolor{darkred}{RGB}{140,0,0}
\definecolor{darkblue}{RGB}{0,0,200}
\usepackage[breaklinks,colorlinks=true,citecolor=darkblue,linkcolor=darkred,urlcolor=darkred]{hyperref}
\usepackage{tcolorbox}
\usepackage{algorithm}
\usepackage{algpseudocode}
\usepackage{enumitem}

\theoremstyle{plain}
\newtheorem{theorem}{Theorem}
\newtheorem{lemma}{Lemma}
\newtheorem{proposition}{Proposition}

\newtheorem{definition}{Definition}
\theoremstyle{remark}

\labelformat{equation}{(#1)}

\newcommand{\Abar}{\bar{\mathcal A}}
\newcommand{\SA}{\mathcal S\times\mathcal A}
\newcommand{\E}{\mathbb E}
\newcommand{\1}{\mathbf 1}
\newcommand{\norm}[1]{\left\lVert #1 \right\rVert}

\title{Adaptive state-action abstractions via rate-distortion}

\author{
  Fernando E.~Rosas\\
  Department of Informatics, University of Sussex\\
  Department of Brain Science, Imperial College London\\
  Centre for Eudaimonia and Human Flourishing, University of Oxford\\
  \texttt{f.rosas@sussex.ac.uk} \\
}

\begin{document}
\maketitle

\begin{abstract}
When learning to walk, infants seem to address a coarse version of the problem first~---~stay upright, reach the caregiver~---~and refine it only when further practice at that resolution stops paying off. 
Reinforcement learning offers multiple techniques for building simple versions of complex tasks, but lacks general principles for how to dynamically adjust the granularity of these abstractions during learning. 
This paper proposes one such principle: \emph{refine the abstraction as soon as the learning error within it becomes comparable to the error induced by the abstraction itself}. 
Here, we investigate one way of formalising this principle via a performance certificate that decomposes value error into two terms: a learning error bound captured by a Bellman residual, and an abstraction error bound given by a bisimulation metric.
The resulting switching strategy is implemented by soft state-action abstractions built from rate-distortion principles,
whose resolution along state and action axes can be continuously adjusted.
We validate this construction in a range of tabular settings, showing that near-optimal performance can be achieved under substantial lossy compression of state and action information.
\end{abstract}

\section{Introduction}

Good control rarely requires tracking every detail of the environment~\citep{konidaris2019necessity,ho2019value,abel2019theory,allen2023structured}.
For instance, task-irrelevant symmetries can be factored away without hindering the ability of a reinforcement-learning agent to find an optimal policy~\citep{ravindran2004algebraic,van2020mdp}. 
More generally, useful abstractions act like sufficient statistics: they preserve what matters the most and discard distinctions whose effects are unimportant.

But how should an agent decide what information to keep and which to discard?
Classical abstraction theory gives exact answers to this question in special cases.
State bisimulation and MDP homomorphisms identify states that are perfectly equivalent from a planning perspective~\citep{givan2003equivalence,ravindran2004algebraic,li2006towards}.
Bisimulation metrics extend these conditions to quasi-symmetry and almost-bisimilarity providing quantitative performance guarantees~\citep{ferns2004metrics,ferns2011bisimulation,taylor2008bounding,abel2016near}.
These ideas are closely related to lossy compression, which has been formally studied under rate-distortion theory~\citep{shannon1959coding} and the information bottleneck method~\citep{tishby2000information}.

While the literature offers a rich account of \emph{how to build} abstractions, comparatively little is known about \emph{how to adjust} their granularity during learning.
Existing adaptive methods refine abstractions based on the distinguishability of transition estimates or the discovery of verification counterexamples~\citep{ortner2013adaptive,clarke2000cegar,abate2024bisimulation,coppola2025existence}.
However, to my knowledge, no existing adaptive method handles state-action abstraction jointly.
This is not an incidental restriction: state aggregation can factor invariances (i.e., distinctions that do not matter) but fails to capture equivariances, in which the same behavioural function carries different action labels depending on the context~\citep{van2020mdp,taylor2008bounding}.
Equivariance is often the rule rather than the exception --- for instance, in bilateral locomotion the action that functions as `advance the free leg' maps to a right-leg command when the left leg is planted, and to a left-leg command when the right leg is planted.

This paper introduces a formal framework to dynamically adjust the resolution of state-action abstractions. 
Our central technical result is a value-error decomposition that splits the total error into a Bellman residual measuring how well the current abstract problem has been solved, and an abstraction error bounded by a bisimulation metric. 
The decomposition motivates a simple adaptive rule: switch to a finer abstraction as soon as the Bellman residual reaches the scale of the abstraction error. 
We formalise this rule by building a continuous family of soft abstractions through rate-distortion, indexed by a temperature parameter that controls resolution. 
We validate the framework on classic tabular control benchmarks and a \texttt{SysAdmin}~\citep{guestrin2003efficient}
scaling test. 
Results show that the adaptive rule traces meaningful compression-distortion frontiers and achieves near-optimal performance under substantial state and action compression.
The framework further reveals whether state, action, or joint state-action information is the binding constraint on performance for a given task.

Together, these findings yield an adaptive method that refines abstractions only as far as a task requires, and in doing so quantifies how much state and action information that task actually contains. Overall, the main contributions are:
\begin{itemize}[leftmargin=2.5em]
\item \emph{Soft state-action abstractions}. A stochastic generalisation of MDP homomorphisms, equipped with bisimulation metrics that compare state-action pairs.
\item \emph{A learning-abstraction decomposition}. A control certificate that decomposes value error into a Bellman residual within the abstraction and bisimulation distortion of the abstraction itself.
\item \emph{The adaptive abstraction principle}. A refinement rule that ties abstraction granularity to learning progress --- refine only when the residual within the current abstraction reaches the scale of the abstraction error.
\item \emph{State-action abstraction built from rate-distortion principles.} A continuous family of soft abstractions in which the information allocated to states and actions is independently controllable.
\item \emph{A decomposition of task compressibility}. The separable rates make it possible to attribute a task's compressibility to its states, actions, or their interaction.
\end{itemize}

\section{Related Work}
\label{sec:related}

\paragraph{Bisimulation and MDP homomorphisms.}
Abstractions have been studied via state bisimulation~\citep{givan2003equivalence} and MDP homomorphisms
\citep{ravindran2004algebraic}, which give exact conditions under which the value function of a reduced MDP can be lifted to the original one~\citep{li2006towards}.
These ideas have been used to build state-action abstractions that account for equivariant symmetries~\citep{van2020mdp} and have been combined with the options framework~\citep{abel2020value}.
This work extends MDP homomorphisms~\citep{ravindran2004algebraic,van2020mdp} from 
hard state-action abstractions based on coarse-grainings to soft abstractions based on stochastic encoders and decoders, whose granularity is dynamically adjusted during planning.

\paragraph{Bisimulation metrics.}
Work on approximate abstraction generalises exact bisimulation, allowing quantitative error and providing performance bounds~\citep{taylor2008bounding,abel2016near,jiang2015abstraction}.
\citet{ferns2004metrics} introduced state bisimulation pseudometrics as
quantitative relaxations of exact bisimulation (see also \citet{ferns2014bisimulation}).
Approximate homomorphism and lax bisimulation-style metrics extend this by comparing actions through matchings rather than fixed labels
\citep{taylor2008bounding,zhao2022continuous}. 
A survey on abstraction guarantees and approximation-estimation tradeoffs can be found in \citet{jiang2018notes}. 
Also, policy-conditioned relaxations~\citep{castro2020scalable,zhao2022continuous,panangaden2024policy} enable ways to account for equivalence during behaviour.
In deep RL, bisimulation-inspired losses are often optimised from samples~\citep{gelada2019deepmdp,zhang2021learning,castro2021mico,kemertas2021towards,agarwal2021contrastive}. Our work extends bisimulation metrics to soft state-action abstractions, allowing stochastic encoders to compress `quasi-equivariant' relationships.

\paragraph{Information-theoretic compression.}
\citet{abel2019state} formulate state abstraction as lossy compression in apprenticeship learning, and use a Blahut--Arimoto scheme to solve the resulting rate-distortion problem; \citet{biza2021learning} use a variational information bottleneck over state-action pairs to learn discrete state abstractions; and \citet{delgrange2022distillation} distil policies into variational MDP
abstractions. Related work studies deterministic bottlenecks \citep{xu2022wdibs}, bottleneck regularisation for exploration and transfer
\citep{goyal2019infobot,igl2019generalization,clauw2025ibtransfer}, and variational objectives whose optima recover bisimulation relations \citep{freed2025vibes}.
More generally, information-theoretic principles have been used to ask what an agent should learn~\citep{arumugam2021deciding,arumugam2021value,arumugam2022value,arumugam2022between,arumugam2022rate}, and bounded-rational control has studied related tradeoffs between information and reward
\citep{polani2009information,tishby2011information,rubin2012trading}.
Our framework builds abstractions using the rate-distortion principle while introducing key differences in the object and the distortion: it compresses state-action control structure rather than only states, and uses bisimulation metrics as distortion rather than expert demonstration.

\paragraph{Adaptive refinement.}
Incremental abstraction takes place during partition refinement when computing bisimulation
quotients \citep{givan2003equivalence}, as well as in earlier model-reduction heuristics \citep{dean1997model}, but this has not been used for online calibration. 
\citet{ortner2013adaptive} has studied online state aggregation based on statistical uncertainty about rewards and transitions during learning. 
Methods such as CEGAR \citep{clarke2000cegar}, data-driven bisimulation learning
\citep{abate2024bisimulation}, and multi-resolution bisimulation constructions
\citep{coppola2025existence} also refine abstractions when a behavioural or verification condition fails.
Our refinement procedure is qualitatively different: we refine neither based on transition estimates nor on verification counterexamples, but instead focus on whether the Bellman residual of the abstract planner reaches the error induced by the abstraction.

\section{Soft state-action abstractions}
\label{sec:soft}

Consider a MDP $(\mathcal S,\mathcal A,P,r,\gamma)$, with finite state space $\mathcal S$, finite action space $\mathcal A$, transition kernel $P(s'|s,a)$, reward $r(s,a)$, and discount factor $\gamma\in[0,1)$~\citep{SuttonBarto2018}.
For a bounded value function $V:\mathcal S\to\mathbb R$, the one-step backup operator is
\begin{equation}
(BV)(s,a)
:=
r(s,a)+\gamma\E_{s'\sim P(\cdot\mid s,a)}\big[V(s')\big].
\label{eq:b-backup}
\end{equation}
The Bellman optimality operator is
$(TV)(s):=\max_{a\in\mathcal A}(BV)(s,a)$.

In many situations the MDP may be too hard to solve directly. 
A common way to simplify the problem is to
compress states, actions, or state-action pairs via coarse-graining mappings. 
Instead, here we consider a more general approach based on
\emph{stochastic coarse-grainings} specified by
\begin{align}
    \nu_S:\mathcal S\to\Delta(\bar{\mathcal S})
    \quad\text{(state encoder),}\quad
    \nu_A:\mathcal S\times\mathcal A\to\Delta(\bar{\mathcal A})
    \quad\text{(action encoder),}
    \\
    \eta_S:\bar{\mathcal S}\to\Delta(\mathcal S)
    \quad\text{(state decoder),}\quad
    \eta_A:\bar{\mathcal S}\times\bar{\mathcal A}\to\Delta(\mathcal A)
    \quad\text{(action decoder),}
\end{align}
where $\bar{\mathcal S}$ and $\bar{\mathcal{A}}$ are sets of abstract states and actions. 
We denote the joint encoder and decoder as
$\nu(\bar s,\bar a|s,a) := \nu_S(\bar s|s) \,\nu_A(\bar a|s,a)
\quad\text{and}\quad
\eta(s,a|\bar s,\bar a) := \eta_S(s|\bar s)\,\eta_A(a|\bar s,\bar a)$.

\begin{definition}
Given an MDP $(\mathcal S,\mathcal A,P,r,\gamma)$, the \textbf{soft state-action abstraction} induced by $\nu_S$, $\nu_A$, $\eta_S$, and $\eta_A$ is the MDP $\big(\bar{\mathcal S},\bar{\mathcal A}(\bar s),\bar P,\bar r,\gamma\big)$ with rewards and dynamics given by
\begin{align}
\bar r(\bar s,\bar a)
&:=
\sum_{s\in\mathcal S}\sum_{a\in\mathcal A}
r(s,a)\,\eta(s,a|\bar{s},\bar{a})
\quad\text{and}
\label{eq:soft-sa-reward}
\\
\bar P(\bar s'\mid \bar s,\bar a)
&:=
\sum_{s\in\mathcal S}\sum_{a\in\mathcal A}\sum_{s'\in\mathcal S}
\nu_S(\bar s'|s')\,P(s'\mid s,a)\,\eta(s,a|\bar s,\bar a),
\label{eq:soft-sa-kernel}
\end{align}
and
$\bar{\mathcal A}(\bar s)
:=
\left\{
\bar a\in\bar{\mathcal A} :
\exists a\in\mathcal A, \exists s\in\mathcal S, \nu(\bar s,\bar a|s,a)>0\right\}$
being the admissible actions at $\bar s\in\bar{\mathcal S}$.
\end{definition}

Soft state-action abstractions generalise state bisimulation and other
approaches to build abstractions (see \autoref{app:soft-sa-bisim}).
The Bellman optimality operator of a soft state-action abstraction is
\begin{equation}
\big(\bar T\bar V\big)(\bar s)
:=
\max_{\bar a\in\Abar(\bar s)}
\Big\{
\bar r(\bar s,\bar a)
+\gamma\E_{\bar s'\sim\bar P(\cdot\mid\bar s,\bar a)}\big[\bar V(\bar s')\big]
\Big\}.
\label{eq:closed-abstract}
\end{equation}
The encoder and decoder induce a `lifting' (from concrete to abstract) and `grounding' (from abstract to concrete) operators
\begin{equation}
\big(LV\big)(\bar s)=\sum_{s\in\mathcal S} V(s)\eta_S(s\mid\bar s)
\quad\text{and}\quad
\big(\Gamma\bar V\big)(s)=\sum_{\bar s\in\bar{\mathcal S}}\bar V(\bar s)\nu_S(\bar s\mid s).
\label{eq:lift-ground-value}
\end{equation}

Soft abstractions are a stochastic extension of traditional MDP homomorphisms~\citep{ravindran2004algebraic,van2020mdp}, which correspond to a deterministic encoder
    $\nu(\bar s,\bar a\mid s,a)=\1\{f(s,a)=(\bar s,\bar a)\}$
    and decoder
    $\eta(s,a\mid\bar s,\bar a)=\1\{g(\bar s,\bar a)=(s,a)\}$,
where
\begin{equation}
    f(s,a)=\big(f_S(s),f_A(s,a)\big)
    \quad\text{with}\quad
    f_S:\mathcal S\to\bar{\mathcal S}
    \quad\text{and}\quad
    f_A:\mathcal S\times\mathcal A\to\bar{\mathcal A}
\end{equation}
are coarse-graining functions, and $g(\bar s,\bar a)=\big(g_S(\bar s),g_A(\bar s,\bar a)\big)\in\mathcal S\times\mathcal A$ is a representative selector.
For this scenario, the abstract MDP has
rewards and transitions given by
\begin{equation}
\bar r(\bar s,\bar a)=r\big(g(\bar s,\bar a)\big)
\quad\text{and}\quad
\bar P\big(\bar s' \mid \bar s,\bar a\big)
=
\sum_{s'\in\mathcal S:f_S(s')=\bar s'}
P\big(s'\mid g(\bar s,\bar a)\big).
\end{equation}
The lifting and grounding operators become
$\big(L V)(\bar s) = V\big(g_S(\bar s)\big)$
and
$\big(\Gamma\bar V\big)(s) = \bar V\big(f_S(s)\big)$.
For a more extended comparison with other abstraction frameworks, see \autoref{app:soft-sa-bisim}.

\section{Learning and abstraction errors}
\label{sec:bounds}

\subsection{Assessing distortion between state-action pairs}

To compare state-action pairs, it is useful to consider distortion pseudometrics\footnote{Pseudometrics are symmetric and satisfy the triangle inequality, but may give $d(x,y)=0$ for $x\neq y$.} over $\mathcal S\times\mathcal A$ that measure how different two state-action pairs are for planning purposes.

\begin{definition}
A distortion $d_{\mathcal V}$ is said to be \textbf{Bellman-compatible} with respect to a collection of value functions $\mathcal V$ if it satisfies
\begin{equation}
d_{\mathcal V}\big((s,a),(s',a')\big)
\geq
\sup_{V\in\mathcal V}
\big|
(BV)(s,a)-(BV)(s',a')
\big|.
\label{eq:v-bellman-compatible}
\end{equation}
\end{definition}

The tightest Bellman-compatible distortion for a given class of value functions $\mathcal V$ is often of the form~~\citep{ferns2004metrics,ferns2011bisimulation} 
\begin{equation}
d_{\rho}\big((s,a),(s',a')\big)
:=
|r(s,a)-r(s',a')|
\;+\;
\gamma\,W_{1,\rho}\Big(P(\cdot\mid s,a),P(\cdot\mid s',a')
\Big),
\label{eq:v-distortion-upper-bound}
\end{equation}
where $W_{1,\rho}$ is the Wasserstein/Kantorovich metric with 
base metric $\rho$. 
For instance, $d_{\rho_\text{ind}}$ with $\rho_\text{ind}(s,u)=\1\{s=u\}$ 
(so that $W_{1,\rho_\text{ind}}$ is the total variation distance) 
is the tightest Bellman-compatible distortion for bounded value functions $\mathcal{V}=\{V:\|V\|_\infty\le 1\}$. 
Another important distortion is given by the fixed-point construction 
by \cite{ferns2011bisimulation}, which satisfies $\rho^*=d_{\rho^*}$. 
The partition induced by the zeros of Ferns' distortion 
corresponds to the minimal bisimulation of the MDP. 
For a discussion of this and other distortions, see \autoref{app:value-compatible-distortions}.

Distortion can be used to estimate the effect of errors induced by the simplifications associated with an abstract MDP. Consider the composition of the lifting and grounding operators, as defined in \autoref{eq:lift-ground-value}, which gives the `round-trip' operator $K:=\Gamma L$ with kernel
\begin{equation}
\kappa(s',a'\mid s,a)
=
\sum_{\bar s\in \bar{\mathcal S}}\sum_{\bar a\in\bar{\mathcal A}}\eta(s',a'\mid\bar s,\bar a)\nu(\bar s,\bar a\mid s,a).
\label{eq:pair-round-trip}
\end{equation}
By marginalising actions, the state round-trip kernel is
$\kappa_S(s'\mid s)
:=
\sum_{\bar s\in \bar S}\eta_S(s'\mid\bar s)\nu_S(\bar s\mid s)$.
Then, the expected distortion introduced by the round-trip can be measured by
\begin{equation}
D_\mathrm{E}(s,a)
:=
\mathbb E_{(S',A')\sim \kappa(\cdot\mid s,a)}
\Big[
d_\mathcal{V}\big((s,a),(S',A')\big)\Big].
\label{eq:roundtrip-defect}
\end{equation}

\subsection{Decomposing errors in optimal value estimation}
\label{sec:bound_and_principle}

We now use the above ideas to derive a decomposition of value error in terms of learning and abstraction effects. For this, we first introduce a new operator that will be useful for our derivations.
\begin{definition}
For a given encoder $\nu$ and decoder $\eta$, the corresponding \textbf{deformed Bellman operator} is
$\bar T^\dag
:=
LT\Gamma$,
which acts on abstract value functions $\bar V$ as
$(\bar T^\dag\bar V)(\bar s)
=
\sum_{s\in\mathcal S}\eta_S(s\mid \bar s)\,(T\Gamma\bar V)(s)$.
\end{definition}
While $\bar T^\dag$ and $\bar T$ act on the same space and use the same state encoder and decoder, they differ in whether the maximum is taken before or after the lifting. 
That said, they do have some similar properties, as shown next.

\begin{lemma}
$\bar T^\dag$ is a $\gamma$-contraction in $\|\cdot\|_\infty$
and hence has a unique fixed point. (Proof in App.~\ref{proof:fixed_point})
\label{lemma:T_eta_fixedpoint}
\end{lemma}
As a second ingredient, consider the geometry of the state-action space induced by the distortion $d$. One way to assess it is to calculate its diameter $\Delta_\text{max}:=\sup_{s,a}\sup_{s',a'} d\big((s,a),(s',a')\big)$.
A subtler way to assess this is given next.
\begin{definition}
\label{def:abstraction_error}
The \textbf{abstraction error bound} is given by
\begin{equation}
\Delta_\mathrm{H}(\kappa_S;d)
:=
\sup_{s\in\mathcal S}
\mathbb E_{S'\sim \kappa_S(\cdot\mid s)}
\big[
D_\mathrm{H}(s,S')
\big],
\label{eq:v-hausdorff-defect}
\end{equation}
where $D_\mathrm{H}(s,u):=
\max\big\{
\sup_{a}\inf_{b}d\big((s,a),(u,b)\big),
\;
\sup_{b}\inf_{a}d\big((s,a),(u,b)\big)
\big\}$.
\end{definition}

The Hausdorff distance $D_\text{H}(s,u)$ provides a tighter estimate than state bisimulation distance $D_B(s,u):=\max_a d\big((s,a),(u,a)\big)$ \citep{taylor2008bounding,taylor2008lax}, as it compares each action with its best matching action at the other state.
For deterministic abstractions, the abstraction error bound reduces to
$\Delta_\text{H}(\kappa_S;d) = \sup_{s\in\mathcal S} D_H(s,g_S\circ f_S(s))$.

With all this in place, we can now formulate the following bound.

\begin{theorem}[Learning--abstraction decomposition]
\label{teo:main}
If $\Gamma\bar V\in\mathcal V$, then
\begin{equation}
\norm{\Gamma\bar V-V^*}_\infty
\le
\frac{1}{1-\gamma}
\Big[
\underbrace{\|\bar T^\dag\bar V-\bar V\|_\infty}_{\text{learning error}}
+
\underbrace{\Delta_H(\kappa_S;d_{\mathcal V})}_{\text{abstraction error}}
\Big].
\label{eq:main-bound}
\end{equation}
(Proof in Appendix~\ref{app:proof_theorem}).
\end{theorem}
\autoref{teo:main} reveals that errors made while performing optimal value estimation on an abstract MDP can be decomposed into two types: errors within the abstract MDP, and errors due to the abstraction itself. This leads to a simple yet powerful idea:

\begin{tcolorbox}[colback=gray!10, colframe=black!60, before upper={\parindent12pt}]
\noindent
\textbf{Adaptive abstraction principle}: Begin learning within a coarse abstraction, and refine it as soon as the learning error becomes smaller than the abstraction error.
\end{tcolorbox}

The next section provides a formalisation of this principle using rate-distortion theory.

\section{Adaptive abstractions based on rate-distortion theory}
\label{sec:rd}

We now exploit the adaptive abstraction principle by building a family of abstractions of increasing degree of granularity using rate-distortion theory.

\subsection{Building abstractions from rate-distortion theory}

To use rate-distortion theory, one first needs to specify a suitable distortion function. While $\Delta_H$ would be an attractive candidate, Hausdorff metrics are often very expensive to calculate.
Following standard rate-distortion practice, let us instead take $\mu\in\Delta(\SA)$ to be a replay, design, or
occupancy measure, and consider the $\mu$-averaged distortion
\begin{equation}
\overline\Delta_\mu(\nu,g)
:=
\mathbb E_{(S,A)\sim\mu,\,(\bar S,\bar A)\sim\nu(\cdot\mid S,A)}
\Big[d_{\mathcal V}\big((S,A),g(\bar S,\bar A)\big)\Big].
\label{eq:mu-decoder-defect}
\end{equation}
\autoref{app:mu-vs-uniform} shows under what conditions a $\mu$-averaged distortion can still assess the worst-case. Additionally, Appendix~\ref{app:deterministic_decoder} shows that the choice of having a deterministic decoder in \autoref{eq:mu-decoder-defect} does not imply lack of generality.

The resulting rate-distortion problem is as follows:
\begin{equation}
    \min_{\nu,g} \overline\Delta_\mu(\nu,g)
\quad \text{subject to}\quad
I_{\mu}(\bar S;S) \leq c_1
\quad\text{and}\quad
I_\mu(\bar A;S,A\mid \bar S) \leq c_2,
\end{equation}
where $c_1$ and $c_2$ are bounds on the amount of information allowed to go into abstract states and actions, respectively.
A Lagrangian formulation of the problem is
\begin{equation}
\min_{\nu,g}
\Big\{
I_{\mu}(\bar S;S)
+
\lambda I_\mu(\bar A;S,A\mid \bar S)
+
\beta \overline\Delta_\mu(\nu,g)
\Big\}.
\label{eq:rd-objective}
\end{equation}
Here, small $\beta$ emphasises compression and therefore yields coarse soft
abstractions; large $\beta$ emphasises distortion and therefore yields
finer ones.
Moreover, $I_{\mu}(\bar S;S)$ and $I_{\mu}(\bar A;S,A|\bar S)$ disentangle the amount of information about the state and action retained by the encoder, respectively. Large $\lambda$ encourages compressing action information, and $\lambda=1$ means that state and action bits are weighted equally; indeed $I_{\mu}(\bar S;S)+ I_\mu(\bar A;S,A\mid \bar S) = I(\bar S,\bar A;S,A)$.

Several familiar abstraction frameworks appear as constrained or limiting cases
of \autoref{eq:rd-objective}:
\begin{itemize}[leftmargin=3.5em]
\item If the encoder $\nu$ is deterministic, then the objective reduces
to a hard state-action partition. In the zero-distortion
limit it recovers state bisimulation and MDP homomorphisms.

\item If actions are preserved, then only the state encoder is learned and the
objective then reduces to classical state abstraction.

\item If $\lambda=0$, the action encoder can perform state-dependent action matching, yielding an analogue of lax bisimulation.

\item If $\beta\to\infty$ at fixed alphabet size, the rate term becomes
negligible and the problem approaches minimum-distortion clustering of
state-action pairs, analogous to $K$-medoids.
\end{itemize}
In contrast, varying $\beta$ and $\lambda$ turns static bisimulation clustering into an adaptive refinement procedure. For a detailed discussion on these comparisons, see \autoref{app:RD_generalises}.

\paragraph{Optimisation.} For given $\beta$, $\lambda$, abstract alphabets
$\bar{\mathcal S}$ and $\bar{\mathcal A}$, and distortion metric,
\autoref{eq:rd-objective} can be optimised by a generalised alternating Blahut--Arimoto scheme~\citep{blahut1972computation,arimoto1972algorithm}. The update alternates between refreshing the abstract
marginals, updating the state encoder with the current abstract-action marginal
frozen, updating the tied action encoder, and choosing blockwise decoder
representatives. The technical details are provided in \autoref{app:flat}.

\subsection{Rate-distortion for adaptive abstractions}
\label{sec:algorithm}

We can now exploit the adaptive abstraction principle (\autoref{sec:bound_and_principle}) using families of abstractions with varying granularity built via rate-distortion, indexed by the resolution parameter $\beta$.

\paragraph{Algorithm sketch.}
The resulting adaptive continuation method is:
\begin{enumerate}[leftmargin=3.5em]
\item Estimate a Bellman-compatible state-action distortion $d$. In practice, this may be via a lax-bisimulation surrogate, a MICo-style sample-based similarity, or another critic/model-based Bellman proxy \citep{castro2021mico,zhao2022continuous}.
\item Fix a finite abstract state alphabet $\bar{\mathcal S}$, abstract action
supports $\bar{\mathcal A}(\bar s)$, an initial decoded codebook
$g^{(0)}\in\mathcal G$, a conditional-rate multiplier $\lambda$, and a small
starting temperature $\beta_0$.
\item Solve the rate-distortion problem outlined in \autoref{eq:rd-objective} for $\beta=\beta_0$ using the methods described in \autoref{app:flat}, store the distortion $\bar \Delta_{\beta_0}$, and build or update the value operator $\bar T^\dag$.
\item Plan or learn on the current abstract problem while monitoring the  error $\|\bar T^\dag(\bar V) - \bar V\|$.
\item When the learning error has fallen to the scale of the current $\bar \Delta_\beta$, increase $\beta$ to the next value on the continuation schedule, warm-start the encoders
and decoded codebook, and solve the new rate-distortion problem.
\item Stop when the desired value tolerance, compute budget, or memory budget is reached.
\end{enumerate}

The resulting sequence $\{(\nu_{S,\beta},\nu_{A,\beta},g_\beta)\}_{\beta}$ corresponds to a Pareto family of soft state-action abstractions with guarantees of minimal distortion $\overline\Delta_\mu$ for given amounts of state-action information.
Note that this does not necessarily result in a nested family of hard quotients.
If a hard abstraction is needed for interpretation or deployment, one can set $\beta\to\infty$ and instead regulate the abstraction by progressively decreasing the alphabet sizes $|\bar S|$ and $|\bar A|$, as developed by~\cite{slonim1999agglomerative}.

\section{Case studies}
\label{sec:experiments}

This section presents experiments used to evaluate the proposed framework for adaptive abstraction.
These experiments investigate three questions: 
whether the rate-distortion objective produces a usable compression-distortion frontier;
whether the adaptive rule can navigate this frontier during planning;
and how the selected abstractions compare with classic bisimulation and MDP homomorphisms.
All experiments use $\lambda=1$ and a simplified
flat rate-distortion solver described in \autoref{app:flat}; additional
environment and evaluation details are given in \autoref{app:experiments}.

\subsection{Tabular control benchmarks}

\begin{figure}[t]
\centering
\includegraphics{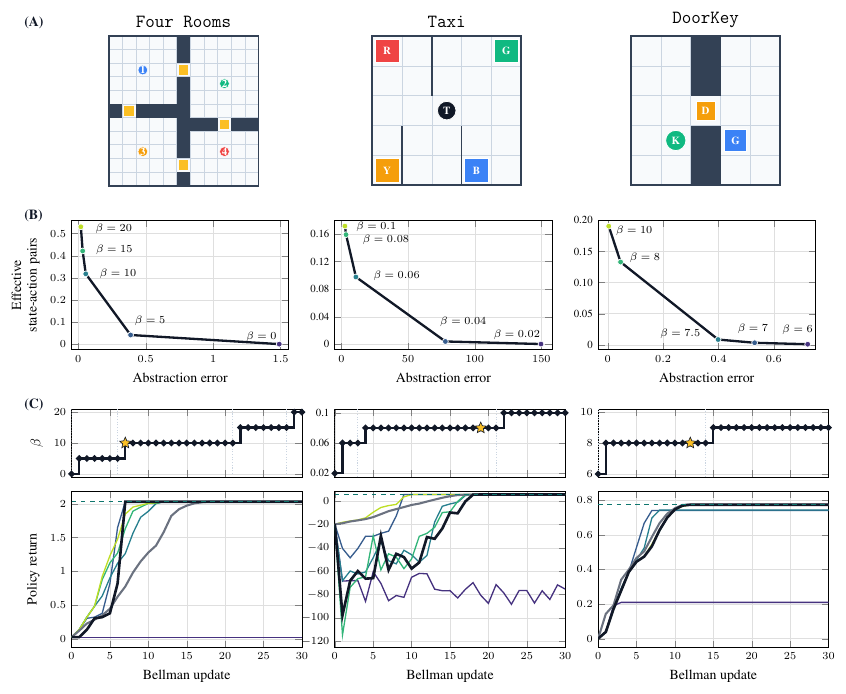}
\caption{
\small\textbf{Adaptive abstractions on tabular control benchmarks}.
\textit{Top}: description of the domains.
\textit{Middle}: compression-distortion frontiers; increasing $\beta$ reduces empirical abstraction error while increasing the  effective number of abstract state-action pairs.
\textit{Bottom}: adaptive trajectory of $\beta$ and policy-return traces.
Bold curves represent the adaptive approach and coloured curves correspond to fixed abstractions. Dashed horizontal lines mark optimal performance.
}
\label{fig:toy}
\end{figure}

We first apply the adaptive abstraction scheme (\autoref{sec:algorithm}) to
three well-known control benchmarks: \texttt{Four Rooms}~\citep{sutton1999between},
\texttt{Taxi}~\citep{dietterich2000hierarchical}, and a fully observable
\texttt{DoorKey} variant inspired by \texttt{MiniGrid}~\citep{chevalier2023minigrid} (see \autoref{fig:toy}).
The domains were chosen because they expose different forms of compressible
control structure. \texttt{Four Rooms} is a navigation problem with spatial
regularity and bottlenecks, so it is used to test state compression. \texttt{Taxi} adds object-oriented task variables,
where pickup and dropoff are relevant only in special contexts, making action
compression suitable. \texttt{DoorKey} combines navigation with a
task-phase dependency: the agent must acquire a key before opening the door and
reaching the goal, so useful compression involves both state identity and local
action relevance.

Results show that a rate-distortion family of abstractions can be successfully built for these three settings: increasing $\beta$ reduces abstraction
error while increasing the effective number of abstract state-action symbols (\autoref{fig:toy}).
Interestingly, the adaptive rule uses this frontier differently across tasks.
In \texttt{Four Rooms} and \texttt{DoorKey}, the scheme settles at an intermediate resolution that already recovers near-optimal performance. In contrast, the coarsest abstractions in \texttt{Taxi} are too lossy and performance improves only at abstractions of relatively high rate.

Our framework allows us to probe what kind of compression takes place in these scenarios.
Analyses at the first adaptive checkpoint whose policy attains optimal performance
show that the compression profiles match the structure of the tasks: \texttt{Four Rooms} compresses most heavily the state factor, \texttt{Taxi} compresses primarily actions, and \texttt{DoorKey} compresses both the task phase and the locally appropriate
action set (see \autoref{table1}).
Overall, adaptive abstractions were found to retain substantially less effective state-action information than the exact bisimulation and MDP homomorphisms.
This is not a contradiction:
bisimulation preserves full value information, whereas the adaptive
criterion allows errors that may be irrelevant for constructing the optimal policy.

\begin{table}[h]
\centering
\caption{Compression diagnostics}
\label{table1}
\begin{tabular}{lcccc}
\toprule
Domain & $|\mathcal S|\!\times\!|\mathcal A|$ & State bisimulation & MDP homomorphism  & Adaptive (state$\times$action)\\
\midrule
\texttt{FourRooms} & $1664$ & $1.0$ & $0.952$ & $0.318=0.368\times0.864$ \\
\texttt{Taxi} & $3006$ & $1.0$ & $0.333$ & $0.159=0.665\times 0.239$ \\
\texttt{DoorKey} & $1165$ & $1.0$ & $0.201$ & $0.133=0.373\times0.357$ \\
\bottomrule
\end{tabular}
\caption*{\footnotesize Compression rates are normalised by $|\mathcal S\!\times\!\mathcal A|$; the compression rate of the adaptive approach is calculated as $2^{I(\bar S,\bar A;S,A)}=2^{I(\bar S;S)}\times2^{I(\bar A;A,S|\bar S)}$, and subcomponents are normalised by $|\mathcal S|$ and $|\mathcal A|$, respectively. The adaptive point is the first checkpoint whose grounded policy attains the optimal performance. More details can be found in \autoref{app:experiments}.}
\end{table}

\subsection{SysAdmin scaling}

\begin{figure}[t]
\centering
\includegraphics{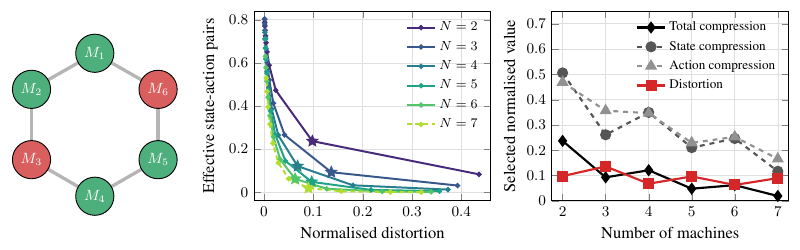}
\caption{
\small
\textbf{Adaptive abstractions on SysAdmin.}
\textit{Left}: Diagram of the scenario.
\textit{Middle}: Rate-distortion frontiers for various numbers of machines; stars mark the resolution at which the abstraction attains near-optimal performance.
\textit{Right}: Compression and distortion of the near-optimal policies as $N$ grows, with the effective size decomposed into state and action contributions.
}
\label{fig:sysadmin}
\end{figure}

We next use a ring-structured \texttt{SysAdmin} benchmark \citep{guestrin2003efficient}
to investigate how our approach scales with problem size.
In this benchmark, the goal is to keep a set of machines working (see~\autoref{fig:sysadmin}). The environment's state is a binary vector recording which
machines are operational; each timestep either does nothing or reboots one machine,
and failures are more likely when neighbouring machines are down (see Appendix~\ref{app:sysadmin}). Thus, a \texttt{SysAdmin} instance with $N$ machines has \(2^N\) states and \(N+1\) actions, and dynamics are generated by local ring interactions.
This makes the domain useful for testing whether our approach captures local control structure rather than memorising the full state-action specification.

Results show that larger systems exhibit more pronounced rate-distortion curves, which suggests that they allow for more compression with less distortion  (see \autoref{fig:sysadmin}).
Furthermore, the abstraction at which near-optimal performance (\(99\%\) of the full-resolution return) is achieved maintains
roughly stable normalised distortion, while retaining a decreasing fraction of
the concrete state-action information. This suggests that the learned codes effectively exploit the local ring structure of the dynamics rather than memorising global
configurations.

We also investigated to what degree the learned near-optimal
encoder provides human-readable ring features (for details of methodology, see Appendix~\ref{app:sysadmin}).
Results are shown in \autoref{tab:sysadmin-factor-representatives}.
On the state side, the largest mutual informations are
with local failure motifs and global load variables, such as triplet
histograms, pair histograms, failed count, and largest failure run.
On the action side, the strongest factors describe the local context of the target
machine, including its five-site neighbourhood, its failure-run size, and
whether it is isolated.
Overall, the probes suggest that abstractions are organised around local outage structure
and locally relevant interventions rather than generic low-rank compression.

\begin{table}[h]
\centering
\caption{Representative SysAdmin factor probes}
\label{tab:sysadmin-factor-representatives}
\begin{tabular}{llrr|llrr}
\toprule
Factor & Type & Bits & Relevance & Factor & Type & Bits & Relevance \\
\midrule
Triplet hist. & State & 1.196 & 0.442 & Target 5-bit & Action & 0.430 & 0.156 \\
Pair hist. & State & 1.153 & 0.447 & Target run size & Action & 0.301 & 0.155 \\
Failed count & State & 0.968 & 0.466 & Target index & Action & 0.289 & 0.140 \\
Max fail run & State & 0.818 & 0.402 & Isolated target & Action & 0.205 & 0.180 \\
\bottomrule
\end{tabular}
\caption*{\footnotesize Information reports the mean raw mutual information in bits across $N=2,\dots,7$. Relevance reports the mean normalised information. The left block lists state factors for $I(F;\bar S)$; the right block lists action factors for $I(F;\bar A\mid \bar S)$.}
\end{table}

\section{Conclusion}

What is the essential information that an agent needs to solve a task?
The approach presented here addresses this question by tuning abstraction granularity
according to two sources of error: the difficulty of solving an abstract problem
and the error introduced by the abstraction itself.
This idea is formalised using soft state-action abstractions built from rate-distortion
principles, which lead to the identification of a value error bound that distinguishes learning error as a Bellman residual and abstraction error as bisimulation distortion.
This yields a simple adaptive rule: plan at the current resolution while the residual dominates, and refine once it reaches the abstraction floor.

Experiments in tabular domains support this principle. In \texttt{Four
Rooms}, \texttt{Taxi}, \texttt{DoorKey}, and \texttt{SysAdmin},
the rate-distortion objective yielded a useful compression-vs-distortion trade-off, and
the adaptive rule selected abstractions that recovered optimal
performance while retaining substantially fewer effective state-action pairs.
Moreover, the separation of state and action information rates makes the
resulting abstractions diagnostically useful --- they indicate if compression
comes mainly from state irrelevance (as in \texttt{Four Rooms}), action matching (as in \texttt{Taxi}), or both (as in \texttt{DoorKey}).

This work focuses on tabular problems to provide a proof-of-principle of
the proposed framework that can cleanly highlight its fundamental principles.
A natural next step is to extend the adaptive
principle to model-free settings, where the learning term can be estimated from
sampled Bellman or TD errors and the abstraction term from sample-based
behavioural metrics or latent-MDP losses
\citep{gelada2019deepmdp,castro2021mico,zhang2021learning,agarwal2021contrastive}.
Another interesting extension would be to combine this framework
with options, which would connect state-action compression at a single decision step
with temporal abstraction over action sequences
\citep{sutton1999between,bacon2017optioncritic,abel2020value}.

Future work may also use adaptive rate-distortion abstractions to enhance the
interpretability of deep reinforcement learning agents.
Building upon the simplicity biases of deep learning systems~\citep{valle2018deep,huh2021low,shai2026transformers},
this could be used to track how
the effective state-action rate changes during training, and whether refinement
events align with improvements in performance.
This also connects to a long-standing observation in cognitive psychology and developmental neuroscience, which states that infants appear to navigate complex tasks by first solving coarse versions and refining only when the coarse solution stops being adequate~\citep{rosch1976basic,newport1990maturational,adolph2008learning}.
Whether AI systems or infants implement a rule similar to the one proposed in this work, or arrive at a similar coarse-to-fine trajectory through different mechanisms, is an empirical question that we hope this work may help to address.

\bibliographystyle{unsrtnat}

\newpage
\appendix

\section{Proofs}
\label{app:proofs}

\subsection{Proof of \autoref{lemma:T_eta_fixedpoint}}
\begin{proof}
  \label{proof:fixed_point}
For two abstract value functions $\bar V_1,\bar V_2$, the following holds:
\begin{align*}
\|\bar T^\dag\bar V_1-\bar T^\dag\bar V_2\|_\infty
&=
\|LT\Gamma\bar V_1-LT\Gamma\bar V_2\|_\infty\\
&\le
\|T\Gamma\bar V_1-T\Gamma\bar V_2\|_\infty\\
&\le
\gamma\|\Gamma\bar V_1-\Gamma\bar V_2\|_\infty\\
&\le
\gamma\|\bar V_1-\bar V_2\|_\infty.
\end{align*}
Above, the first and third inequalities are due to \autoref{lemma:small}, and
the second inequality follows from \autoref{lemma:bellman_contraction}.
This proves that $\bar T^\dag$ is a $\gamma$-contraction in $L^\infty$.

To conclude the proof, note that since the abstract state space is
finite, the space of bounded abstract value functions is complete, and Banach's
fixed-point theorem gives a unique fixed point.
\end{proof}

The above proof uses the following basic lemmas, whose proof we include for completeness.
\begin{lemma}
\label{lemma:small}
    The lifting and grounding operators satisfy
    \begin{equation}
        \|L V_1 - L V_2\|_\infty\leq \| V_1 - V_2\|_\infty
        \quad\text{and}\quad
        \|\Gamma\bar V_1 - \Gamma\bar V_2\|_\infty\leq \| \bar V_1 - \bar V_2\|_\infty.
    \end{equation}
\end{lemma}
\begin{proof}
    A direct calculation shows that
\begin{equation}
\big|(LV_1)(\bar s)-(LV_2)(\bar s)\big|
\le
\sum_s\eta_S(s\mid\bar s)\big|V_1(s)-V_2(s)\big|
\le
\|V_1-V_2\|_\infty,
\end{equation}
which implies that $\|L V_1 - L V_2\|_\infty\leq \| V_1 - V_2\|_\infty$. An analogous derivation shows that $\|\Gamma\bar V_1 - \Gamma\bar V_2\|_\infty\leq \| \bar V_1 - \bar V_2\|_\infty$.

\end{proof}

\begin{lemma}
\label{lemma:bellman_contraction}
    The Bellman operator $T$ satisfies
    $\|TV_1-TV_2\|_\infty\le\gamma \|V_1-V_2\|_\infty$.
\end{lemma}
\begin{proof}
Using the fact that
$\left|\max_a x_a-\max_a y_a\right|
\le
\max_a |x_a-y_a|$ one can obtain that
\begin{align*}
|(TV_1)(s)-(TV_2)(s)|
&\le
\max_{a\in\mathcal A}
\gamma
\left|
\sum_{s'\in\mathcal S}
P(s'\mid s,a)\big(V_1(s')-V_2(s')\big)
\right| \\
&\le
\gamma
\max_{a\in\mathcal A}
\sum_{s'\in\mathcal S}
P(s'\mid s,a)
|V_1(s')-V_2(s')| \\
&\le
\gamma \|V_1-V_2\|_\infty.
\end{align*}
Taking the supremum over \(s\in\mathcal S\) gives the desired result.
\end{proof}

\subsection{Proof of~\autoref{teo:main}}
\label{app:proof_theorem}

\begin{proof}
  Applying \autoref{lemma:concrete_residual} to \(W=\Gamma\bar V\) gives
  \[
  \|\Gamma\bar V-V^*\|_\infty
  \le
  \frac{1}{1-\gamma}
  \|\Gamma\bar V - T\Gamma\bar V\|_\infty.
  \]
  Then, a direct calculation using the triangle inequality shows that
  \begin{align*}
  \|\Gamma\bar V - T\Gamma\bar V\|_\infty
  &\le
  \|\Gamma\bar V-\Gamma\bar T^\dag\bar V\|_\infty
  +
  \|\Gamma\bar T^\dag\bar V-T\Gamma\bar V\|_\infty\\
  &\le
  \|\bar V - \bar T^\dag\bar V\|_\infty
  +
  \Delta_H(\kappa_S;d_{\mathcal V}).
  \end{align*}
  Above, the last inequality follows from \autoref{lemma:small} and \autoref{prop:intertwining}.
\end{proof}

\begin{lemma}
\label{lemma:concrete_residual}
For every bounded concrete value function \(W:\mathcal S\to\mathbb R\),
\begin{equation}
\|W-V^*\|_\infty
\le
\frac{1}{1-\gamma}\|TW-W\|_\infty .
\label{eq:concrete-residual-bound}
\end{equation}
\end{lemma}
\begin{proof}
Since \(V^*=TV^*\), the triangle inequality and
\autoref{lemma:bellman_contraction} give
\[
\|W-V^*\|_\infty
=
\|W-TV^*\|_\infty
\le
\|W-TW\|_\infty+\|TW-TV^*\|_\infty
\le
\|W-TW\|_\infty+\gamma\|W-V^*\|_\infty .
\]
Rearranging proves \autoref{eq:concrete-residual-bound}.
\end{proof}

\begin{proposition}
\label{prop:intertwining}
If $\Gamma\bar V\in\mathcal V$, then
$\norm{T\Gamma\bar V-\Gamma\bar T^\dag\bar V}_\infty
\le
\Delta_H(\kappa_S;d_{\mathcal V})$.
\end{proposition}

\begin{proof}
Let us first show that Bellman compatibility implies that
\begin{equation}
|(TW)(s)-(TW)(u)|\le D_H(s,u)
\quad\text{for any}\quad
W\in\mathcal V, s,u\in\mathcal S.
\label{eq:value-lax-ineq}
\end{equation}
To see this, suppose first that $(TW)(s)\ge (TW)(u)$. Then
let $a_*\in\arg\max_a(BW)(s,a)$. Since the action set is finite,
\begin{align*}
(TW)(s)-(TW)(u)
&=
\max_a(BW)(s,a)-\max_b(BW)(u,b)\\
&=
(BW)(s,a_*)-\max_b(BW)(u,b)\\
&=
\min_b\big|(BW)(s,a_*)-(BW)(u,b)\big|\\
&\le
\max_a\min_b\big|(BW)(s,a)-(BW)(u,b)\big|\\
&\le
\max_a\min_b d_{\mathcal V}\big((s,a),(u,b)\big)\\
&\le
D_H(s,u).
\end{align*}
The case $(TW)(u)\ge (TW)(s)$ is obtained in an analogous way using the other directed Hausdorff term in the definition of $D_H$.

Now set $W=\Gamma\bar V$. Since $\bar T^\dag=LT\Gamma$ and
$K_S:=\Gamma L$ has kernel $\kappa_S$,
\[
\Gamma\bar T^\dag\bar V
=
\Gamma LT\Gamma\bar V
=
K_S(TW).
\]
Thus, for each $s$,
\[
(\Gamma\bar T^\dag\bar V)(s)
=
\E_{U\sim\kappa_S(\cdot\mid s)}[(TW)(U)].
\]
Then, using \autoref{eq:value-lax-ineq} and assuming $W\in\mathcal V$, then
\begin{align*}
\big|
(T\Gamma\bar V)(s)-(\Gamma\bar T^\dag\bar V)(s)
\big|
&=
\big|
(TW)(s)-\E_{U\sim\kappa_S(\cdot\mid s)}[(TW)(U)]
\big|\\
&\le
\E_{U\sim\kappa_S(\cdot\mid s)}
\Big[\big|(TW)(s)-(TW)(U)\big|\Big]\\
&\le
\E_{U\sim\kappa_S(\cdot\mid s)}\big[D_H(s,U)\big].
\end{align*}
Taking the supremum over $s$ and using the definition of
\(\Delta_H(\kappa_S;d_{\mathcal V})\) gives the desired result.
\end{proof}

\subsection{Proof of generality of \autoref{eq:mu-decoder-defect}}
\label{app:deterministic_decoder}

\begin{lemma}[Deterministic decoders suffice]
\label{lem:deterministic-decoder-suffices}
Suppose that, in the averaged rate-distortion objective, the representative map
$g$ is replaced by an arbitrary stochastic pair decoder
$\chi(\cdot\mid \bar s,\bar a)\in\Delta(\mathcal S\times\mathcal A)$, and define
\[
\overline\Delta_\mu(\nu,\chi)
:=
\mathbb E_{\substack{(S,A)\sim\mu,\,(\bar S,\bar A)\sim\nu(\cdot\mid S,A),\\
(U,B)\sim\chi(\cdot\mid \bar S,\bar A)}}
\Big[d_{\mathcal V}\big((S,A),(U,B)\big)\Big].
\]
Then for every fixed encoder $\nu$ there exists a (deterministic) representative
map $g:\bar{\mathcal S}\times\bar{\mathcal A}\to\mathcal S\times\mathcal A$
such that
\[
\overline\Delta_\mu(\nu,g)
\le
\overline\Delta_\mu(\nu,\chi)
\]
for every stochastic pair decoder $\chi$. Consequently, restricting
\autoref{eq:mu-decoder-defect} to deterministic representatives causes no loss
of optimality in the rate-distortion objective \autoref{eq:rd-objective}.
\end{lemma}

\begin{proof}
Fix an encoder $\nu$ and let
\[
\pi_\nu(\bar s,\bar a)
:=
\sum_{s,a}\mu(s,a)\,\nu(\bar s,\bar a\mid s,a)
\]
be the induced marginal over abstract pairs. For every abstract pair with
$\pi_\nu(\bar s,\bar a)>0$, define the posterior
\[
\nu_\mu(s,a\mid \bar s,\bar a)
:=
\frac{\mu(s,a)\,\nu(\bar s,\bar a\mid s,a)}
{\pi_\nu(\bar s,\bar a)}.
\]
Then the averaged defect can be written blockwise as
\begin{align*}
\overline\Delta_\mu(\nu,\chi)
&=
\sum_{\bar s,\bar a}\pi_\nu(\bar s,\bar a)
\sum_{u,b}\chi(u,b\mid \bar s,\bar a)
\sum_{s,a}\nu_\mu(s,a\mid \bar s,\bar a)\,
d_{\mathcal V}\big((s,a),(u,b)\big).
\end{align*}
For fixed $(\bar s,\bar a)$, the inner expression is a linear functional of the
probability vector $\chi(\cdot\mid \bar s,\bar a)$ over the simplex
$\Delta(\mathcal S\times\mathcal A)$. A linear functional on a simplex attains
its minimum at an extreme point, hence at a point mass. Therefore choose
\[
g(\bar s,\bar a)
\in
\arg\min_{(u,b)\in\mathcal S\times\mathcal A}
\sum_{s,a}\nu_\mu(s,a\mid \bar s,\bar a)\,
d_{\mathcal V}\big((s,a),(u,b)\big)
\]
on each positive-mass block, and choose $g(\bar s,\bar a)$ arbitrarily on
zero-mass blocks. The resulting deterministic decoder satisfies
\[
\overline\Delta_\mu(\nu,g)
\le
\overline\Delta_\mu(\nu,\chi)
\]
for every stochastic pair decoder $\chi$.

\end{proof}

\section{From state bisimulation to soft state-action abstractions}
\label{app:soft-sa-bisim}

This appendix provides a brief overview of various approaches for building abstractions in
reinforcement learning, tracing a path from classic bisimulation to soft
state-action abstractions as defined in \autoref{sec:soft}.

\subsection{Bisimulation}

A classic approach to building lossless state abstractions is \emph{bisimulation}
\citep{dean1997model,givan2003equivalence}.

\begin{definition}\label{def:state-bisim}
Let $M=(\mathcal S,\mathcal A,P,r,\gamma)$ be an MDP.
A \textbf{bisimulation} is an equivalence relation
$\sim$ on $\mathcal S$ such that, whenever
$s\sim s'$, the following conditions hold for every action $a\in\mathcal A$:
\begin{align}
r(s,a) &= r(s',a), \label{eq:state-bisim-reward} \\
\sum_{u\in C} P(u\mid s,a) &= \sum_{u\in C}P(u\mid s',a)
\quad
\text{for every equivalence class } C\in \mathcal S / {\sim}. \label{eq:state-bisim-kernel}
\end{align}
\end{definition}

Thus, two states are bisimilar when every action yields the same immediate
reward and the same next-state law after one forgets distinctions inside the
equivalence classes. Here, the equivalence classes of bisimilar states are built
deterministically, primitive action labels are compared directly, and the
conditions on reward
and transition equalities must hold exactly. 
Bisimulation builds upon the notion of \emph{lumpability} of Markov chains, 
which highlights coarse-grainings of Markov chains that
result in processes that are also Markovian~\citep{kemeny1960finite}. 
The term `bisimulation' comes from the idea that bisimilar MDPs can simulate
each other, in the sense that they are behaviourally
equivalent~\citep{li2006towards} and hence everything that can be done in one can also be done in the other.

\subsection{MDP homomorphisms}

State bisimulation leverages the idea of relaxing distinctions between environmental states; 
thus, a natural extension is to abstract state-action pairs jointly.
Indeed, Equations \ref{eq:state-bisim-reward} and \ref{eq:state-bisim-kernel} suggest 
state-action pairs $(s,a)$ as natural objects to coarse-grain, 
as they determine immediate reward and next-state transitions. 
Thus, a state-action abstraction assigns `concrete' state-action 
pairs to `abstract' pairs via coarse-graining mappings~\citep{ravindran2004algebraic,taylor2008bounding,abel2016near,abel2020value,zhao2022continuous}.
A standard way to operationalise this idea is through MDP homomorphisms, as defined next.

\begin{definition}[MDP homomorphism]
\label{def:mdp-homomorphism}
An \textbf{MDP homomorphism} from a MDP
\(M=(\mathcal S,\mathcal A,P,r,\gamma)\) to another MDP
\(\bar M=(\bar{\mathcal S},\bar{\mathcal A}(\bar s),\bar P,\bar r,\gamma)\)
is a pair of surjective maps
$f(s,a)=\big(f_S(s),f_A(s,a)\big)$
with
$f_S:\mathcal S\to\bar{\mathcal S}$ and
$f_A(s,a)\in\bar{\mathcal A}(f_S(s))$, such that for all
\(s\in\mathcal S\), \(a\in\mathcal A\), and
\(\bar u\in\bar{\mathcal S}\),
\begin{align}
r(s,a)
&=
\bar r\big(f_S(s),f_A(s,a)\big),
\label{eq:hom-reward}
\\
\bar P\big(\bar u\mid f_S(s),f_A(s,a)\big)
&=
\sum_{s':f_S(s')=\bar u}P(s'\mid s,a).
\label{eq:hom-kernel}
\end{align}
\end{definition}

Equivalently, one may define a homomorphic quotient directly from a map
$f=(f_S,f_A)$ without first specifying \(\bar r\) and \(\bar P\). The required condition is that concrete pairs with the same abstract label
have the same one-step control effect at the abstract level:
\begin{align}
f(s,a)=f(t,b)
&\Longrightarrow
r(s,a)=r(t,b), \label{eq:sa-hom-reward}\\
f(s,a)=f(t,b)
&\Longrightarrow
\sum_{s':f_S(s')=\bar u}P(s'\mid s,a)
=
\sum_{s':f_S(s')=\bar u}P(s'\mid t,b)
\quad
\text{for every }\bar u\in\bar{\mathcal S}.
\label{eq:sa-hom-kernel}
\end{align}
When equations \ref{eq:sa-hom-reward} and \ref{eq:sa-hom-kernel} hold, the abstract
reward and transition in \autoref{def:mdp-homomorphism} are obtained by
choosing any representative \((s,a)\) of an abstract pair
\((\bar s,\bar a)\) and setting
\[
\bar r(\bar s,\bar a)=r(s,a),
\qquad
\bar P(\bar u\mid\bar s,\bar a)
=
\sum_{s':f_S(s')=\bar u}P(s'\mid s,a).
\]
The conditions ensure that this does not depend on which representative is
chosen.

One can think of an MDP homomorphism as a state-action analogue of
bisimulation.\footnote{Nonetheless, the literature calls this object
an MDP homomorphism, action abstraction, or value-preserving state-action
abstraction, rather than `state-action bisimulation'~\citep{ravindran2004algebraic,taylor2008bounding,abel2020value,zhao2022continuous}.} 
Indeed, the map $f$ induces the following
equivalence relation on $\mathcal S\times\mathcal A$:
\[
(s,a)\sim_{SA}(t,b)
\quad\Longleftrightarrow\quad
f(s,a)=f(t,b).
\]
Equations~\ref{eq:sa-hom-reward} and \ref{eq:sa-hom-kernel} establish that
state-action-equivalent pairs have equal rewards and equal transition mass into
every state block, being the pair-level analogue of
\autoref{def:state-bisim}.
If the action component is forced to preserve its labels 
(i.e., $f_A(s,a)=a$), 
then the construction reduces to standard bisimulation.

\subsection{Action-role matching via symmetry-based homomorphisms}

Classical state bisimulation compares two
states by asking whether every concrete action has the same effect in both
states; in contrast, the action component of an MDP homomorphism can depend on the
state. This raises an interesting question for state-action abstraction: under what conditions can one say that  different primitive actions play the same control role in different concrete states?

Symmetry-based homomorphisms give an elegant answer to this 
question. These are a particular case of MDP homomorphisms, 
in which states and actions are matched via a structural 
symmetry~\citep{ravindran2004algebraic,van2020mdp}. 

Concretely, let a group
$G$ act on $\mathcal S$ through state maps $\sigma_g:\mathcal S\to\mathcal S$,
and let $\rho_g^s:\mathcal A\to\mathcal A$ be the corresponding action
relabelling at state $s$.
One can show that this group yields a MDP homomorphism if the following conditions hold~\citep{van2020mdp}:
\[
r\big(\sigma_g(s),\rho_g^s(a)\big)=r(s,a),
\qquad
P\big(\sigma_g C\mid \sigma_g(s),\rho_g^s(a)\big)
=
P(C\mid s,a),
\]
where $C$ is an orbit generated by the action $\sigma_g$ and $\sigma_g C=\{\sigma_g(u):u\in C\}$. 

Thus, using the resulting orbits as equivalence classes
produces a MDP homomorphism. The fact that $\rho_g^s$ may depend on the
current state is what allows symmetry-based homomorphisms to capture
equivariant relationships, where the same functional role is expressed by
different primitive action labels in different states.

\subsection{Relaxations of state bisimulation and MDP homomorphisms}

The various notions above can be relaxed along several axes --- one can replace equality by a metric error, replace same-label action comparison
by action-role matching, replace worst-case all-action requirements by
policy-conditioned ones, or replace exact model preservation by learned
surrogate losses. We review some of these approaches next.

\paragraph{Bisimulation metrics.}
The canonical relaxation of bisimulation is the fixed-point
pseudometric presented by \citet{ferns2004metrics,ferns2011bisimulation} (see \autoref{app:value-compatible-distortions}). 
Instead of asking whether two states have identical rewards and identical transition mass
into every equivalence class, it assigns a distance by comparing the same
primitive action at both states. 
To build this metric, one considers the following functional:
\[
\mathcal F(\rho)(s,t)
:=
\max_{a\in\mathcal A}
\Big\{
c_r \big|r(s,a)-r(t,a)\big|
+
c_p W_\rho\big(P(\cdot\mid s,a),P(\cdot\mid t,a)\big)
\Big\},
\]
where \(c_r,c_p\) are nonnegative constants, and \(W_\rho\) is the Kantorovich/Wasserstein distance induced by the current
state pseudometric \(\rho\). 
It can be shown~\citep{ferns2004metrics,ferns2011bisimulation} 
that the smallest fixed point of \(\mathcal F\) gives zero 
distance exactly on bisimulation classes, while also providing a 
graded notion of almost-bisimilarity. Note that this is primarily 
a \emph{state abstraction} method: actions are used to test states, 
but they are not coarse-grained.

\paragraph{Lax bisimulation and approximate homomorphisms.}
A lax bisimulation metric~\citep{taylor2008bounding} is used to 
relax the notion of MDP homomorphism in an analogous way 
as Ferns' bisimulation metric is used to relax state bisimulation. 
Given a state pseudometric \(\rho:\mathcal S\times\mathcal S\to\mathbb R_+\) and nonnegative weights
\(c_r,c_p\), a state-action pair discrepancy can be written as
\[
\delta_\rho\big((s,a),(t,b)\big)
:=
c_r |r(s,a)-r(t,b)|
+
c_p W_\rho\big(P(\cdot\mid s,a),P(\cdot\mid t,b)\big).
\]
This quantity compares two concrete pairs even when \(a\) and \(b\) are
different primitive labels.
Bisimulation uses this quantity to create a state-only pseudometric by considering the action that gives the worst-case divergence.
A subtler approach is taken by lax bisimulation, which builds a state metric through a symmetric Hausdorff matching over actions:
\[
D_\text{H}(\rho)(s,t)
:=
\max\left\{
\sup_{a\in\mathcal A}\inf_{b\in\mathcal A}
\delta_\rho\big((s,a),(t,b)\big),
\;
\sup_{b\in\mathcal A}\inf_{a\in\mathcal A}
\delta_\rho\big((s,a),(t,b)\big)
\right\}.
\]
The fixed point of this Hausdorff-style operator has zeros corresponding
to the state classes of exact homomorphic quotients.
At nonzero distance, the same construction gives
a quantitative approximate-homomorphism error
\citep{taylor2008bounding,zhao2022continuous}.
Thus, the lax bisimulation metric is
state-level, but it is induced by state-action pair comparisons and
Hausdorff action matching. Thus, it can be used to relax MDP homomorphisms in general, and symmetry-based homomorphisms in particular.

\paragraph{Behavioural and learned relaxations.}
Another kind of relaxation is to preserve only distinctions relevant to a policy, a
value criterion, or the data distribution under which the representation is
learned. Policy-conditioned similarities and value-preserving state-action
abstractions weaken exact bisimulation or exact homomorphism by preserving the
aspects of the dynamics needed for a particular behavioural objective rather
than the full controlled transition structure
\citep{abel2020value,castro2020scalable,panangaden2024policy}. In deep RL, the
exact reward and transition comparisons are often replaced by sample-based
losses or latent-model objectives, as in DeepMDP-style latent models, MICo,
robust bisimulation metric learning, contrastive behavioural similarities, and
related objectives
\citep{gelada2019deepmdp,zhang2021learning,castro2021mico,kemertas2021towards,agarwal2021contrastive,liao2023pitfalls}.
These methods often do not explicitly construct a quotient MDP, but they aim to learn representations in which irrelevant behavioural distinctions are small.

\subsection{Soft state-action coarse-grainings}

A final step is to relax the coarse-graining map itself. State bisimulation,
homomorphisms, and their metric relaxations still presuppose `hard' assignments:
a concrete state or state-action pair either belongs to a block or it does not.
Soft abstractions replace this hard membership by stochastic encoders, in the
same broad family as stochastic model reduction, variational abstractions, and
information-bottleneck approaches to representation learning
\citep{dean1997model,abel2019state,biza2021learning,delgrange2022distillation,goyal2019infobot,igl2019generalization,xu2022wdibs}.

For state-action abstraction, the stochastic encoder used in
\autoref{sec:soft} has the factorised form
\[
\nu(\bar s,\bar a\mid s,a)
=
\nu_S(\bar s\mid s)\,\nu_A(\bar a\mid s,a),
\]
where $\nu_S:\mathcal S\to\Delta(\bar{\mathcal S})$ and
$\nu_A:\mathcal S\times\mathcal A\to\Delta(\bar{\mathcal A})$. This keeps the
state/action separation of MDP homomorphisms, while allowing each concrete pair
to distribute mass over several abstract pairs. 
Because there are no hard coarse-grainings anymore, 
this does not lead to equivalence classes. 
Instead, the concrete and abstract dynamics should \emph{intertwine}: evolving
in the concrete MDP and then encoding should agree, in expectation, with
encoding first and then evolving in the abstract MDP.

The construction presented in \autoref{sec:soft} is one way to instantiate such a
soft abstraction. In addition to the encoder $\nu$, it introduces a decoder
\(\eta\), and defines \(\bar r\) and \(\bar P\) by the decoder-averaged
formulas (Equations \ref{eq:soft-sa-reward} and \ref{eq:soft-sa-kernel}).
The phrase `soft state-action abstraction' is meant to emphasize the specific
line of descent here: it is a stochastic version of the pair-level behavioural
equivalence encoded by MDP homomorphisms. 

A comparison of the various abstraction approaches discussed here is provided in \autoref{tab:abstraction-comparison}.

\begin{table}[h!]
\centering
\caption{Comparison of abstraction approaches in reinforcement learning}
\begin{tabular}{@{}p{4.3cm}p{3.0cm}p{1.3cm}p{3.0cm}@{}}
\toprule
\textbf{Name} & \textbf{Type} & \textbf{Exact} & \textbf{Coarse-graining} \\
\midrule
Bisimulation & State abstraction & Yes & Deterministic \\
\midrule
MDP homomorphism & State-action\newline abstraction & Yes & Deterministic \\
\midrule
Bisimulation metric & State abstraction & No & Deterministic \\
\midrule
Lax bisimulation & Action-aware state\newline abstraction & No & Deterministic \\
\midrule
Soft state-action abstractions & State-action\newline abstraction & No & Stochastic \\
\bottomrule
\end{tabular}\label{tab:abstraction-comparison}
\end{table}

\section{Bellman-compatible state-action distortions}
\label{app:value-compatible-distortions}

Let $x=(s,a)$ and $y=(u,b)$ be two state-action pairs and use
$P_x:=P(\cdot\mid x)$ and $P_y:=P(\cdot\mid y)$ as shorthand notation for the transition kernels. 
For a class of value functions $\mathcal V$, the worst Bellman-backup discrepancy is bounded by
\begin{align}
\sup_{V\in\mathcal V}
\big|
(BV)(x)-(BV)(y)
\big|
&=
\sup_{V\in\mathcal V}
\left|
r(x)-r(y)
\;+\;
\gamma
\sum_{s'\in\mathcal S} V(s')\,\Big(P(s'|x) - P(s'|y)\Big)
\right|
\notag
\\
&\le
|r(x)-r(y)|
\;+\;
\gamma
D_{\mathcal V}(P_x,P_y),
\label{eq:value-compatible-decomposition}
\end{align}
where
\begin{equation}
D_{\mathcal V}(p,q)
:=
\sup_{V\in\mathcal V}
\left|
  \sum_{s'\in\mathcal S} V(s')\,\Big( p(s') - q(s')\Big)
\right|.
\label{eq:value-induced-ipm}
\end{equation}
If $\mathcal V$ is closed under sign changes (i.e., $V\in\mathcal V$ implies $-V\in\mathcal V$), 
then the upper bound in
\autoref{eq:value-compatible-decomposition} is tight whenever
$D_{\mathcal V}(P_x,P_y)<\infty$. Indeed, for any $\varepsilon>0$ choose
$V_\varepsilon\in\mathcal V$ such that
\[
\left|\sum_{s'\in\mathcal S} V_\varepsilon(s')\,\Big( P(s'|x) - P(s'|y)\Big)\right|
\ge D_{\mathcal V}(P_x,P_y)-\varepsilon.
\]
Using the fact that one can always replace
$V_\varepsilon$ by $-V_\varepsilon$ if needed to align the sign of the transition
term with $r(x)-r(y)$, one can guarantee that
\[
\sup_{V\in\mathcal V}
\left|
r(x)-r(y)+\gamma\sum_{s'\in\mathcal S} V(s')\,\Big(P(s'|x) - P(s'|y)\Big)
\right|
\ge
|r(x)-r(y)|
+\gamma\Big(D_{\mathcal V}(P_x,P_y)-\varepsilon\Big).
\]
Letting $\varepsilon\downarrow0$ gives equality. This motivates the canonical
Bellman-compatible upper bound
\begin{equation}
d_{\mathcal V}^*(x,y)
:=
|r(x)-r(y)|
\;+\;
\gamma
D_{\mathcal V}(P_x,P_y).
\label{eq:value-compatible-symmetric}
\end{equation}
In the following we consider special cases of $\mathcal V$.

\paragraph{Bounded $L^\infty$ value classes.}
If $\mathcal V_{\infty,C}:=\{V:\|V\|_\infty\le C\}$, then
\begin{align}
D_{\mathcal V_{\infty,C}}(p,q)
&=
C\sup_{\|f\|_\infty\le 1}
\left|
  \sum_{s\in\mathcal S} f(s)\Big(p(s) - q(s)\Big)
  \right|
\notag
=
C\sum_{z\in\mathcal S}|p(z)-q(z)|,
\label{eq:linfty-ipm-tv}
\end{align}
where the second equality follows by taking
$f(z)=\operatorname{sign}(p(z)-q(z))$. With the convention
$d_{\mathrm{TV}}(p,q):=\sum_{s\in\mathcal S}|p(s)-q(s)|$, this gives
\begin{equation}
d_{\mathcal V_{\infty,C}}^*(x,y)
=
|r(x)-r(y)|
\;+\;
\gamma\,C\,
d_{\mathrm{TV}}\!\big(P_x,P_y\big).
\end{equation}
This choice gives a uniform estimate, but ignores any geometry on $\mathcal S$ and therefore 
provides a conservative estimate.

\paragraph{Bounded Lipschitz classes.}
Let $\rho$ be a pseudometric on $\mathcal S$ and set
\[
\mathcal V_{\mathrm{Lip}(\rho),C}
:=
\{V:\|V\|_{\mathrm{Lip},\rho}\le C\},
\quad\text{where}\quad
\|V\|_{\mathrm{Lip},\rho}
:=
\sup_{\rho(s,u)>0}\frac{|V(s)-V(u)|}{\rho(s,u)}.
\]
Since $P_x$ and $P_y$ are probability measures with equal total mass, additive
constants in $V$ do not affect $\sum_{s\in\mathcal S} V(s)\,\Big(P(s|x) - P(s|y)\Big)$. 
Then, one can show that
\begin{align}
D_{\mathcal V_{\mathrm{Lip}(\rho),C}}(P_x,P_y)
&=
C
\sup_{\|f\|_{\mathrm{Lip},\rho}\le1}
\left|\sum_{s\in\mathcal S} f(s)\,\Big(P(s|x) - P(s|y)\Big)\right|
=
C\,W_{1,\rho}(P_x,P_y),
\label{eq:lip-ipm-wasserstein}
\end{align}
where $W_{1,\rho}$ is the Wasserstein distance with ground cost $\rho$ defined as 
\[
W_{1,\rho}(p,q)
:=
\min_{\zeta\in\Pi(p,q)}
\sum_{s,u\in\mathcal S}\rho(s,u)\,\zeta(s,u)
\]
and $\Pi(p,q)$ is the set of couplings with marginals $p$ and $q$
\[
\Pi(p,q)
:=
\Big\{\zeta: \zeta(s,u)\ge 0, \sum_{s\in\mathcal S}\zeta(s,u)=q(u), \sum_{u\in\mathcal S}\zeta(s,u)=p(s)\Big\}.
\]
Above, the second equality in \autoref{eq:lip-ipm-wasserstein} follows from the Kantorovich-Rubinstein duality.
Thus
\[
d_{\mathcal V_{\mathrm{Lip}(\rho),C}}^*(x,y)
=
|r(x)-r(y)|
\;+\;
\gamma\,C\,
W_{1,\rho}\!\big(P_x,P_y\big).
\]
The finite-state setting makes the first-moment condition automatic. In more
general spaces, the same formula requires finite first moments under $\rho$.

\paragraph{Ferns' fixed-point metric.}
The two distortions considered above evaluate dynamics considering only one step ahead. 
Thus, they may state that two state-action pairs are similar without taking into account their long-term consequences.
A classic construction that goes beyond this is Ferns' fixed-point metric \citep{ferns2004metrics,ferns2011bisimulation}, which we introduce next.

Given a state pseudometric $\rho:\mathcal S\times\mathcal S\to[0,\infty)$, define the associated distortion
\begin{equation}
d_\rho\big((s,a),(u,b)\big)
:=
|r(s,a)-r(u,b)|
\;+\;
\gamma W_{1,\rho}\!\big(P(\cdot\mid s,a),P(\cdot\mid u,b)\big).
\label{eq:rho-pair-distortion}
\end{equation}
The operator introduced in \citet{ferns2004metrics,ferns2011bisimulation} compares states using the worst-case action matching:
\begin{equation}
(\mathcal F\rho)(s,u)
:=
\max_{a\in\mathcal A}
d_\rho\big((s,a),(u,a)\big).
\label{eq:ferns-state-operator}
\end{equation}
It can be shown that $\mathcal F$ is a $\gamma$-contraction in $L^\infty$ for finite MDPs, and thus it has a fixed point.  
At this fixed point,
\[
\rho_B(s,u)
=
\max_{a\in\mathcal A}
\Big\{
|r(s,a)-r(u,a)|
+\gamma W_{1,\rho_B}\!\big(P(\cdot\mid s,a),P(\cdot\mid u,a)\big)
\Big\}.
\]
The zero sets of $\rho_B$ are exactly the state-bisimulation classes: distance
zero forces equal rewards and equal transition mass into each zero-distance
class for every action, and those two equalities conversely make the fixed-point
distance vanish. This is the sense in which the metric used inside the
Wasserstein term is also the metric being solved for.

For state-action abstraction with action relabelling, one can replace the
same-label maximum in \autoref{eq:ferns-state-operator} by the Hausdorff/lax
action matching used in the main text:
\begin{equation}
D_H^\rho(s,u)
:=
\max\left\{
\sup_{a\in\mathcal A}\inf_{b\in\mathcal A}
d_\rho\big((s,a),(u,b)\big),
\;
\sup_{b\in\mathcal A}\inf_{a\in\mathcal A}
d_\rho\big((s,a),(u,b)\big)
\right\}.
\label{eq:lax-hausdorff-state-distance}
\end{equation}
This replaces literal action identity by best action matching across states.
The main bound only requires a Bellman-compatible pair distortion; it can
therefore use either the same-action Ferns metric, the Hausdorff/lax variant,
or a computable surrogate for them.

\section{From \texorpdfstring{$\mu$}{mu}-averaged to worst-case distortion}
\label{app:mu-vs-uniform}

The main bound in \autoref{teo:main} is stated in terms of the worst-case
Hausdorff abstraction distortion
\[
\Delta_H(\kappa_S;d_{\mathcal V})
=
\sup_{s\in\mathcal S}
\mathbb E_{U\sim\kappa_S(\cdot\mid s)}
\big[D_H(s,U)\big].
\]
By contrast, the rate-distortion objective in \autoref{eq:rd-objective}
optimises a $\mu$-averaged reconstruction defect. This appendix makes explicit
what such an average can and cannot guarantee.

For this purpose, let us consider the quantity
$e_{\kappa}(s)
:=
\mathbb E_{U\sim\kappa_S(\cdot\mid s)}
\big[D_H(s,U)\big]$. 
Using it, the worst-case Hausdorff distortion can be expressed as
\begin{equation}
\Delta_H(\kappa_S;d_{\mathcal V})=\sup_{s\in\mathcal S} e_\kappa(s).
\end{equation}
Additionally, for a given distribution $\mu_S\in\Delta(\mathcal S)$
we can also use $e_{\kappa}(s)$ to build the averaged Hausdorff distortion as
\begin{equation}
\Delta^H_{\mu_S}(\kappa_S;d_{\mathcal V})
:=
\mathbb E_{S\sim\mu_S}[e_\kappa(S)].
\label{eq:mu-hausdorff-defect}
\end{equation}
We now study relationships between these two quantities.

\begin{proposition}
\label{prop:mu-to-worst-case}
For every distribution $\mu_S$ on $\mathcal S$,
$\Delta^H_{\mu_S}(\kappa_S;d_{\mathcal V})
\le
\Delta_H(\kappa_S;d_{\mathcal V})$.
If $\mathcal S$ is finite and
$\mu_{\min}:=\min_{s\in\mathcal S}\mu_S(s)>0$,
then
\begin{equation}
\Delta_H(\kappa_S;d_{\mathcal V})
\le
\frac{
\Delta^H_{\mu_S}(\kappa_S;d_{\mathcal V})
}{
\mu_{\min}
}.
\label{eq:worst-via-mu-min}
\end{equation}
\end{proposition}

\begin{proof}
The first inequality follows because an expectation is bounded above by the
supremum of its argument. For the second inequality, choose
$s_*\in\arg\max_s e_\kappa(s)$. Then
\[
\Delta^H_{\mu_S}(\kappa_S;d_{\mathcal V})
=
\sum_s\mu_S(s)e_\kappa(s)
\ge
\mu_S(s_*)e_\kappa(s_*)
\ge
\mu_{\min}\Delta_H(\kappa_S;d_{\mathcal V}).
\]
Rearranging gives \autoref{eq:worst-via-mu-min}.
\end{proof}

The proposition shows that a $\mu$-averaged distortion can 
certify the worst case when $\mu$ assigns sufficient mass to 
each state at which distortion may be large.
Note that the resulting constant may be poor: for a uniform distribution on a finite
state space, \autoref{eq:worst-via-mu-min} gives the factor $|\mathcal S|$.
However, no distribution-free converse is possible without such a coverage condition.

Finally, average distortion \(\overline\Delta_\mu(\nu,g)\) used in
\autoref{eq:mu-decoder-defect} should be read
as a computable surrogate for these averaged defects. It becomes a uniform
certificate only with the corresponding coverage assumptions, and only when the
chosen pair reconstruction controls the Hausdorff state defect relevant to
\autoref{teo:main}. Otherwise, it remains a useful design criterion for building
a family of abstractions.

\section{How rate-distortion generalises abstraction frameworks}
\label{app:RD_generalises}

The rate-distortion formulation presented in \autoref{eq:rd-objective} can be seen as an umbrella construction containing many familiar bisimulation and abstraction constructions as limiting or constrained cases. 
Some of these constructions are discussed in detail in \autoref{app:soft-sa-bisim}.

\begin{enumerate}
\item \emph{Hard abstractions.} If the state and action encoders are
deterministic Dirac kernels
\begin{align}
\nu(\bar s,\bar a\mid s,a)
=\mathbf 1\{f(s,a)=(\bar s,\bar a)\}
\end{align}
for $f(s,a)=\big(f_S(s),f_A(s,a)\big)$ with 
$f_S:\mathcal S\to\bar{\mathcal S}$ and 
$f_A:\mathcal S\times\mathcal A\to\bar{\mathcal A}$, 
then every concrete state $s$ and action $a$ is assigned to
a unique abstract state/action $(\bar s,\bar a)$.
The stochastic rate-distortion problem
then reduces to a hard partition or quotient problem with decoded
representatives. In the zero-distortion case, and with enough abstract symbols,
this is the usual exact quotient/homomorphism setting of stochastic
bisimulation and MDP homomorphisms
\citep{givan2003equivalence,ravindran2004algebraic,taylor2008bounding}.

\item \emph{Approximate hard homomorphisms.} If the encoders are deterministic
but the attained pair distortion is nonzero, then the objective is an
approximate homomorphism problem: each block has a representative decoded state
and decoded actions, and the abstraction error is the averaged
Bellman-compatible mismatch within that block. This is the finite-alphabet,
lossy analogue of approximate MDP homomorphisms and value-preserving
state-action abstractions
\citep{taylor2008bounding,abel2020value,zhao2022continuous}.

\item \emph{Lax bisimulation/action matching.} Setting $\lambda=0$ removes the
penalty for retaining action information inside an abstract state. The action
encoder can then use as much action information as the abstract action alphabet
permits to match state-dependent control roles, so the objective becomes an averaged,
rate-distortion version of the lax-bisimulation geometry. It becomes the lax
bisimulation metric itself only if the distortion is replaced by the
Hausdorff aggregation introduced in \autoref{def:abstraction_error}, and $d_{\mathcal V}$
is chosen to be the corresponding Bellman fixed-point pseudometric (see \autoref{app:soft-sa-bisim}); 
otherwise it is a tractable mean-distortion surrogate built from the same state-action matching idea.

\item \emph{Classical state abstractions.} If actions are not compressed (for
example $\bar{\mathcal A}=\mathcal A$, $\nu_A(\bar a\mid s,a)=\mathbf 1\{\bar a=a\}$
and the action component of $g(\bar s,\bar a)$ is $\bar a$), then the only
learned compression is the state encoder $\nu_S$. In this case, 
the objective reduces to a state-abstraction problem with
action labels preserved, recovering the standard setting studied in state
aggregation, stochastic bisimulation partitions, and the abstraction taxonomy of
\citet{li2006towards}; approximate variants correspond to replacing exact
equivalence by a positive distortion budget
\citep{givan2003equivalence,li2006towards,abel2016near,jiang2018notes}.

\item \emph{Bisimulation clustering for a fixed alphabet.} If $\beta\to\infty$
while the number of abstract states and decoded abstract actions is fixed, the
rate term becomes negligible and the problem becomes minimum-distortion
clustering under a given distortion $d$. With deterministic encoders this is the
state-action analogue of $K$-medoids or partition-around-medoids~\citep{kaufman1990finding}.
If $d$ is a
bisimulation or homomorphism metric, this recovers the usual practice of
clustering states or state-action pairs by behavioural distance; if the alphabet
is large enough to achieve zero distortion, it recovers exact bisimulation
minimisation~\citep{ferns2004metrics,ferns2011bisimulation}.

\item \emph{Deterministic annealing and soft clustering.} Finite $\beta$ is the
soft counterpart of hard bisimulation clustering: blocks are probabilistic,
high-distortion assignments are exponentially suppressed, and increasing
$\beta$ follows a deterministic-annealing path from coarse to fine
abstractions. This is the same optimisation principle behind classical
rate-distortion and information-bottleneck state abstraction, now using a
Bellman-compatible state-action distortion rather than an imitation or
prediction loss.

\end{enumerate}

The formulation also introduces numerous novel settings corresponding to other parameter settings:

\begin{itemize}

\item \emph{Degenerate action collapse is not ordinary state abstraction.} If
instead the action encoder is constant, or each $\bar{\mathcal A}(\bar s)$ has
only one action, the method collapses control choices rather than merely
abstracting states. This can be useful as a deliberately low-capacity controller,
but it should not be identified with classical state abstraction, which normally
keeps the primitive action labels available at every abstract state
\citep{li2006towards,abel2020value}.

\item \emph{Symmetries and action relabellings.} Because $\nu_A$ enters the
distortion term, the objective can learn state-dependent action relabellings and
equivariances. This is the homomorphism intuition behind state-action
abstraction: different concrete actions may play the same behavioural role in
different states, and the abstraction should preserve that role rather than the
literal action name
\citep{ravindran2004algebraic,taylor2008bounding,van2020mdp,zhao2022continuous,panangaden2024policy}.

\item \emph{Adaptive refinement.} Finally, varying $\beta$, the alphabet size,
or the admissible action supports turns static bisimulation clustering into an
adaptive refinement procedure. Classical model minimisation and partition
refinement split until exact behavioural consistency is reached; the present
formulation replaces the hard consistency test by a residual-versus-distortion
scale comparison, allowing the abstraction resolution to track what the learner
can currently estimate.
\end{itemize}

\section{Rate-Distortion Solvers}
\label{app:flat}

Throughout this appendix, write $d$ for the Bellman-compatible pair distortion
$d_{\mathcal V}$. Let $\mathcal G$ denote the set of admissible representative
maps $g:\bar{\mathcal S}\times\bar{\mathcal A}\to\mathcal S\times\mathcal A$,
and let $\mathcal G(\bar s)$ denote the corresponding set of decoder slices
$h:\bar{\mathcal A}(\bar s)\to\mathcal S\times\mathcal A$ at abstract state
$\bar s$.

\subsection{Structured solver}

For any state and action encoders, write their abstract marginals and posterior as
\begin{align}
\pi_{\nu}(\bar s)
&:=
\sum_s \mu_S(s)\nu_S(\bar s\mid s),
\label{eq:ba-pair-marginal}
\\
m_{\nu}(\bar a\mid \bar s)
&:=
\frac{
\sum_{s,a}\mu(s,a)\nu_S(\bar s\mid s)\nu_A(\bar a\mid s,a)
}{
\pi_{\nu}(\bar s)
},
\\
\nu_{\mu_S}(s\mid\bar s)
&:=
\frac{\mu_S(s)\nu_S(\bar s\mid s)}{\pi_{\nu}(\bar s)}.
\end{align}
The conditional marginal and posterior are defined arbitrarily on zero-mass
abstract states. These quantities are recomputed whenever the encoder changes.
Write $\mu_A(\cdot\mid s)$ for the conditional action reference induced by
$\mu$.

For fixed $\beta$, we optimise the objective (\autoref{eq:rd-objective}) using a generalised alternating
Blahut--Arimoto scheme, viewed as alternating minimisation over encoders,
marginals, and decoders
\citep{blahut1972computation,arimoto1972algorithm,csiszar1984information,rose1998deterministic}.
The state-encoder step is the BA update with the current abstract-state
marginal and abstract-action marginal frozen.
Its state cost includes both a
distortion term and the contribution of
$I_\mu(S,A;\bar A\mid\bar S)$. The action-encoder step is the analogous
conditional rate-distortion update for the soft action code. It can be solved
by projected gradient or exponentiated gradient. In the variant where the
action encoder is allowed to condition on the sampled abstract state,
$\nu_A(\bar a\mid s,a,\bar s)$, this subproblem has the usual BA softmax form
with prior $m_\nu(\bar a\mid\bar s)$ and distortion
$d\big((s,a),g^{(n)}(\bar s,\bar a)\big)$. Under the convention of this
paper, where $\nu_A(\bar a\mid s,a)$ is shared across the possible abstract
states, the same step becomes a tied-parameter conditional rate-distortion
update. The decoder step is a blockwise state-action medoid update of the
structured decoder slices $g(\bar s,\cdot)$. See an implementation in \autoref{alg:structured-ba}.

\begin{algorithm}[h!]
\caption{Structured Blahut--Arimoto updates for state-action rate-distortion}
\label{alg:structured-ba}
\begin{algorithmic}[1]
\Require Reference measure $\mu\in\Delta(\mathcal S\times\mathcal A)$, distortion $d$, alphabets $\bar{\mathcal S},\bar{\mathcal A}$, parameters $\beta,\lambda\ge 0$, tolerance $\varepsilon$, initial encoders $\nu_S^{(0)}$, $\nu_A^{(0)}$ and decoder $g^{(0)}$.
\For{$n=0,1,2,\ldots$}
    \State Compute the abstract-state marginal
    \[
    \pi^{(n)}(\bar s)
    =
    \sum_s \mu_S(s)\nu_S^{(n)}(\bar s\mid s).
    \]
    \State Compute the abstract-action marginal
    \[
    m^{(n)}(\bar a\mid\bar s)
    =
    \frac{
    \sum_{s,a}\mu(s,a)\nu_S^{(n)}(\bar s\mid s)
    \nu_A^{(n)}(\bar a\mid s,a)
    }{
    \pi^{(n)}(\bar s)
    }.
    \]

    \State Compute the frozen state cost
    \[
    C_S^{(n)}(s,\bar s)
    =
    \mathbb E_{A\sim\mu_A(\cdot\mid s),\,
    \bar A\sim\nu_A^{(n)}(\cdot\mid s,A)}
    \left[
    \lambda
    \log
    \frac{\nu_A^{(n)}(\bar A\mid s,A)}
    {m^{(n)}(\bar A\mid\bar s)}
    +
    \beta
    d\big((s,A),g^{(n)}(\bar s,\bar A)\big)
    \right].
    \]

    \State Update the state encoder:
    \[
    \nu_S^{(n+1)}(\bar s\mid s)
    =
    \frac{
    \pi^{(n)}(\bar s)
    \exp\{- C_S^{(n)}(s,\bar s)\}
    }{
    \sum_{\bar u}
    \pi^{(n)}(\bar u)
    \exp\{- C_S^{(n)}(s,\bar u)\}
    }.
    \]

    \State Update the action encoder by solving
    \[
    \nu_A^{(n+1)}
    \in
    \arg\min_{\nu_A}
    \left\{
    \lambda I_\mu(S,A;\bar A\mid \bar S)
    +
    \beta\overline\Delta_\mu(\nu_S^{(n+1)},\nu_A,g^{(n)})
    \right\}.
    \]

    \State Refresh the posterior
    \[
    \nu_{\mu_S}^{(n+1)}(s\mid\bar s)
    =
    \frac{\mu_S(s)\nu_S^{(n+1)}(\bar s\mid s)}
    {\sum_u \mu_S(u)\nu_S^{(n+1)}(\bar s\mid u)}.
    \]

    \State Update the decoder blockwise:
    \[
    g^{(n+1)}(\bar s,\cdot)
    \in
    \arg\min_{h\in\mathcal G(\bar s)}
    \sum_{s,a}
    \sum_{\bar a\in\bar{\mathcal A}(\bar s)}
    \mu(s,a)
    \nu_S^{(n+1)}(\bar s\mid s)
    \nu_A^{(n+1)}(\bar a\mid s,a)
    d\big((s,a),h(\bar a)\big).
    \]

    \If{the objective in \eqref{eq:rd-objective} changes by at most $\varepsilon$}
        \State \Return $(\nu_S^{(n+1)},\nu_A^{(n+1)},g^{(n+1)})$
    \EndIf
\EndFor
\end{algorithmic}
\end{algorithm}

\subsection{Simplified flat pairwise solver used in the diagnostics}

With $\lambda=1$ and the joint-encoder simplification described below, the
experiments solve
\begin{equation}
\min_{\nu,\,g\in\mathcal G}
\left\{
I_{\mu}(S,A;\bar S,\bar A)
+
\beta \overline\Delta_\mu(\nu,g)
\right\}.
\end{equation}
Additionally, they introduce the following simplification: instead of optimising $\nu_S$ and $\nu_A$ separately, they focus on $\nu:\mathcal S\times\mathcal A\to \Delta(\bar{\mathcal{S}}\times\bar{\mathcal{A}})$ as a single entity. This allows for a more computationally efficient algorithm, which is described in \autoref{alg:flat-pairwise-ba}.

\begin{algorithm}[h!]
\caption{Flat pairwise Blahut--Arimoto updates used in the diagnostics}
\label{alg:flat-pairwise-ba}
\begin{algorithmic}[1]
\Require Reference measure $\mu\in\Delta(\mathcal S\times\mathcal A)$, distortion $d$, alphabets
$\bar{\mathcal S}$ and $\bar{\mathcal A}$, parameter $\beta\ge 0$, tolerance $\varepsilon$, initial
joint encoder $\nu^{(0)}$ and decoder $g^{(0)}$.
\For{$n=0,1,2,\ldots$}
    \State Compute the abstract marginal
    \[
    \pi^{(n)}(\bar s,\bar a)
    =
    \sum_{s\in\mathcal S}
    \sum_{a\in\mathcal A}
    \mu(s,a)\nu^{(n)}(\bar s,\bar a\mid s,a).
    \]

    \State Compute the frozen pair cost
    \[
    C^{(n)}\big((s,a),(\bar s,\bar a)\big)
    =
    d\big((s,a),g^{(n)}(\bar s,\bar a)\big).
    \]

    \State Update the encoder:
    \[
    \nu^{(n+1)}(\bar s,\bar a\mid s,a)
    =
    \frac{
    \pi^{(n)}(\bar s,\bar a)
    \exp\{-\beta C^{(n)}((s,a),(\bar s,\bar a))\}
    }{
    \sum_{\bar u\in\bar{\mathcal S}}
    \sum_{\bar b\in\bar{\mathcal A}}
    \pi^{(n)}(\bar u,\bar b)
    \exp\{-\beta C^{(n)}((s,a),(\bar u,\bar b))\}
    }.
    \]

    \State Refresh the posterior
    \[
    \nu_{\mu}^{(n+1)}(s,a\mid \bar s,\bar a)
    =
    \frac{
    \mu(s,a)\nu^{(n+1)}(\bar s,\bar a\mid s,a)
    }{
    \sum_{u\in\mathcal S}
    \sum_{b\in\mathcal A}
    \mu(u,b)\nu^{(n+1)}(\bar s,\bar a\mid u,b)
    }.
    \]

    \State Update the decoder:
    \[
    g^{(n+1)}(\bar s,\bar a)
    \in
    \arg\min_{(u,b)\in\mathcal S\times\mathcal A}
    \sum_{s\in\mathcal S}
    \sum_{a\in\mathcal A}
    \nu_{\mu}^{(n+1)}(s,a\mid \bar s,\bar a)
    d\big((s,a),(u,b)\big).
    \]

    \If{the objective changes by at most $\varepsilon$}
        \State \Return $(\nu^{(n+1)},g^{(n+1)})$
    \EndIf
\EndFor
\end{algorithmic}
\end{algorithm}

\section{Experimental details}
\label{app:experiments}

This appendix provides additional information regarding the experiments. 
Code to reproduce the experimental results can be found in \href{https://github.com/ferosas/adaptive-state-action-abstraction}{\textit{github.com/ferosas/adaptive-state-action-abstraction}}.

\subsection{Shared model-based protocol}
\label{app:shared-experiment-protocol}

All experiments in \autoref{sec:experiments} are tabular discounted
MDPs whose transitions and rewards are assumed to be known. 
Thus the transition tensor \(P(\cdot\mid s,a)\), reward function
\(r(s,a)\), Bellman updates, policy evaluation, and state-action distortion
matrices are all computed exactly. No Monte Carlo rollouts are used for the
reported returns. Unless stated otherwise, the discount is \(\gamma=0.95\), the
reference measure in the rate-distortion objective is the uniform measure
\(\mu(s,a)=1/(|\mathcal S||\mathcal A|)\), and the reported abstraction error is
the average decoder distortion
\[
\overline\Delta_\mu(\nu,g)
=
\mathbb E_{(S,A)\sim \mu,\,Z\sim\nu(\cdot\mid S,A)}
\big[d((S,A),g(Z))\big].
\]
The switching rule therefore uses this average distortion proxy, not the
Hausdorff worst-case quantity from \autoref{def:abstraction_error}.

\paragraph{Distortion.}
All experiments build a fixed-point state-action bisimulation pseudometric of the form
\[
d^*\big((s,a),(u,b)\big)
=
|r(s,a)-r(u,b)|
+\gamma W_{d^*}\!\left(P(\cdot\mid s,a),P(\cdot\mid u,b)\right).
\]
Here \(W_{d^*}\) is the finite-state Wasserstein distance with ground
cost \(d^*\). For deterministic domains such as \texttt{Taxi} and
\texttt{DoorKey}, this Wasserstein term reduces to the distance between the
successor states. For stochastic \texttt{SysAdmin}, it is solved exactly by the
dual linear program. \texttt{Four Rooms} uses the same fixed-point object, with
an implementation that exploits the uniform-reset structure to compute it
efficiently.

\paragraph{Rate-distortion fitting.}
The experiments use the flat solver from \autoref{app:flat} with
\(\lambda=1\). Internally this solver fits a single soft encoder
\(\nu(z\mid s,a)\) over concrete state-action pairs and a deterministic
representative \(g(z)\in\mathcal S\times\mathcal A\) for each active code
\(z\). The abstract alphabet cap is set to the full number of concrete
state-action pairs, \(|\mathcal S||\mathcal A|\); after fitting, codes whose
abstract marginal under \(\mu\) is below \(10^{-4}\) are pruned and the encoder
is renormalised. Blahut--Arimoto updates use tolerance \(10^{-6}\). The
per-domain limits on outer and inner alternating updates are given below.

\paragraph{Planning and grounding.}
Planning is performed in \(Q\)-space. The full-resolution baseline starts from
\(Q_0=0\) and applies the concrete optimality update
\[
(FQ)(s,a)
=
r(s,a)
+\gamma\sum_{s'}P(s'| s,a)\max_{a'}Q(s',a').
\]
For a given abstraction, \(\bar Q(\bar s,\bar a)\) can be grounded to concrete state-action pairs via
\[
Q^{\mathrm{gr}}(s,a)
=
\sum_{\bar s,\bar a} \nu(\bar s,\bar a| s,a)\bar Q(\bar s,\bar a),
\]
and the decoder can be used to read the concrete Bellman update back:
\[
(\bar F\bar Q)(\bar s,\bar a)=(FQ^{\mathrm{gr}})(g(\bar s,\bar a)).
\]

For a fixed \(\beta\), traces start from \(\bar Q_0=0\) and repeatedly apply
\(\bar F\). For the adaptive setting, traces start at the coarsest \(\beta\) and 
switch to finner abstractions when the residual
\[
\max_{s\in\mathcal S}
\left|
\max_a Q^{\mathrm{gr}}_{\mathrm{next}}(s,a)
-
\max_a Q^{\mathrm{gr}}(s,a)
\right|
\]
falls below the current \(\overline\Delta_\mu\). When switching, the current
grounded \(Q\)-function is transferred to candidate finner abstractions, 
and the controller skips any later abstraction whose
transferred residual is still below its abstraction error. 

\paragraph{Evaluation.}
The return traces shown in the paper measure the quality of the grounded concrete policy 
against the corresponding number of Bellman updates. At a recorded
checkpoint \(k\), the plotted policy is always the deterministic greedy concrete policy
\[
\pi_k(s)\in\arg\max_{a\in\mathcal A} Q^{\mathrm{gr}}_k(s,a),
\]
with ties broken by the listed action order. The value plotted for this policy
is its exact value in the original concrete MDP. Concretely, we solve
\[
(I-\gamma P^{\pi_k})v^{\pi_k}=r^{\pi_k},
\qquad
r^{\pi_k}(s)=r(s,\pi_k(s)),
\qquad
P^{\pi_k}(s,s')=P(s'\mid s,\pi_k(s)),
\]
and report the uniform state average
\[
\frac{1}{|\mathcal S|}\sum_{s\in\mathcal S}v^{\pi_k}(s),
\]
including the absorbing terminal state when a domain has one.

\paragraph{Information and compression diagnostics.}
All information quantities are computed in bits. For a given abstraction with 
states $\bar s\in\bar{\mathcal S}$ and actions $\bar a\in\bar{\mathcal A}$, 
the effective number of abstract labels is \(2^{I_\mu(S,A;\bar S,\bar A)}\). 
We then report the chain-rule split
\[
I_\mu(S,A;\bar S,\bar A)
=
I_\mu(S,A;\bar S)
+
I_\mu(S,A;\bar A\mid \bar S).
\]
The corresponding information-equivalent sizes are
\(2^{I_\mu(S,A;\bar S)}\) and
\(2^{I_\mu(S,A;\bar A\mid\bar S)}\). These effective sizes are not
the same as the raw alphabet sizes --- they discount unused or imbalanced labels 
and stochastic assignments. 
Rates divide these two effective sizes by \(|\mathcal S|\) and \(|\mathcal A|\),
respectively, and the normalised state-action rate is their product.

\paragraph{Exact bisimulation baselines.}
The state-bisimulation column of \autoref{table1} is computed by exact
partition refinement: states are split until states in the same block have the
same reward and the same transition mass into every current block for every
primitive action. Since this baseline keeps $\bar{\mathcal A}=\mathcal A$, the
normalised state-action rate is
\((\#\text{state equivalence classes})/|\mathcal S|\). The MDP homomorphism column 
is computed from the zero-distance relation of the
fixed-point pair metric \(d\), using connected components of
\(\{((s,a),(u,b)):d((s,a),(u,b))=0\}\) with numerical tolerance
\(10^{-8}\max\{1,\max d\}\). The adaptive point in \autoref{table1} is the
first adaptive checkpoint whose exact grounded return reaches the optimal return.

\subsection{Classic tabular benchmarks}
\label{app:classic-tabular-details}

\paragraph{Four Rooms.}
The \texttt{Four Rooms} experiment uses the classic \(11\times 11\) four-room
layout with four hallways. A state is a pair consisting of the current
cell and the current goal index. The four goal cells are the room centers
\((2,2)\), \((3,8)\), \((8,2)\), and \((8,8)\). 
There are \(104\) playable cells (i.e. discounting walls) 
and four possible goal rooms, giving \(104\times 4=416\) states. The action set is
\(\{\texttt{up},\texttt{right},\texttt{down},\texttt{left}\}\), so
\(|\mathcal S||\mathcal A|=1664\). With probability \(1-\eta\) the
intended move is executed, and with probability \(\eta\) the agent stays in
place; wall collisions also leave the agent in place (experiments use \(\eta=0.10\)). 
If the landing cell is the current goal center, the agent receives reward \(1\) and
the next state is reset uniformly over all cell-goal pairs, including the same
goal. All other rewards are zero. The fixed and adaptive \(\beta\) ladder is
\(\{0,5,10,15,20\}\). The Blahut--Arimoto limit is \(500\) outer iterations and
\(50\) inner iterations per \(\beta\), and the planning budget is \(150\)
full-resolution sweeps with evaluation every sweep.

\paragraph{Taxi.}
The \texttt{Taxi} experiment follows the Taxi-v3 dynamics on the \(5\times 5\)
map with depot locations \(R,G,Y,B\), but adds a single absorbing success state
so that the episodic task is represented as a discounted MDP. Nonterminal states
are \((\text{row},\text{column},\text{passenger location},\text{destination})\),
where the passenger is either at one of the four depots or in the taxi. This
gives \(500\) nonterminal states plus one absorbing state, for \(501\) states
total. The action set is
\(\{\texttt{south},\texttt{north},\texttt{east},\texttt{west},
\texttt{pickup},\texttt{dropoff}\}\), so
\(|\mathcal S||\mathcal A|=3006\). Transitions are deterministic. Each ordinary
step has reward \(-1\), illegal pickup/dropoff has reward \(-10\), and a
successful dropoff has reward \(20\) and transitions to the absorbing state. The
\(\beta\) ladder is
\(\{0.02,0.04,0.06,0.08,0.10\}\). The Blahut--Arimoto limit is \(10\) outer
iterations and \(100\) inner iterations per \(\beta\), and the planning budget
is \(100\) full-resolution sweeps with evaluation every sweep.

\paragraph{DoorKey.}
The \texttt{DoorKey} experiment is a fully observable tabular variant inspired
by MiniGrid. The grid has side length \(5\). An internal wall separates two
rooms; the locked door is at \((2,2)\), the key is at \((1,3)\), and the goal is
at \((3,3)\). The state records the agent position, facing direction, whether
the key is carried, and whether the door is open. Only reachable task phases
are included: key absent with door closed, key carried with door closed, and
key carried with door open, plus one absorbing terminal state. This yields
\(233\) states. The action set is
\(\{\texttt{turn\_left},\texttt{turn\_right},\texttt{forward},
\texttt{pickup},\texttt{toggle}\}\), so \(|\mathcal S||\mathcal A|=1165\).
Transitions are deterministic. The only positive reward is the goal reward,
equal to \(1\), obtained when the agent moves into the goal and enters the
absorbing state. The fixed \(\beta\) ladder is
\(\{6,7,7.5,8,10\}\). The adaptive ladder is the unit-spaced refinement
\(\{6,7,8,9,10\}\). The Blahut--Arimoto limit is \(200\) outer iterations and
\(50\) inner iterations per \(\beta\), and the planning budget is \(100\)
full-resolution sweeps with evaluation every sweep.

\subsection{SysAdmin}
\label{app:sysadmin}

We use a fully observable ring variant of the \texttt{SysAdmin} benchmark
\citep{guestrin2003efficient}. A state \(x\in\{0,1\}^N\) records which machines
are operational, with \(x_i=1\) meaning that machine \(i\) is up. For \(N\ge 3\)
the neighbours of \(i\) are \(i-1\) and \(i+1\) modulo \(N\); for \(N=2\), each
machine has the other machine as its single neighbour. Let \(\mathcal N(i)\)
denote this ring-neighbour set. The action set is
\[
\mathcal A=\{\mathrm{noop},\mathrm{reboot}_1,\ldots,\mathrm{reboot}_N\}.
\]
At each step the controller either does nothing or attempts to reboot one
machine. The reward is the fraction of machines currently up, minus a reboot
cost when a reboot action is used:
\[
r(x,a)
=
\frac{1}{N}\sum_{i=1}^N x_i
-
c_{\mathrm{reboot}}\mathbf 1\{a\neq \mathrm{noop}\}.
\]
The parameters used in all SysAdmin experiments are
\[
p_{\mathrm{base}}=0.95,\qquad
\rho=0.15,\qquad
p_{\mathrm{recover}}=0.05,\qquad
p_{\mathrm{reboot}}=0.95,\qquad
c_{\mathrm{reboot}}=0.2.
\]
Conditional on the current state and action, machines transition independently.
If machine \(i\) is rebooted, then
\(\mathbb P(X'_i=1\mid x,a)=p_{\mathrm{reboot}}\). If it is down and not
rebooted, then
\(\mathbb P(X'_i=1\mid x,a)=p_{\mathrm{recover}}\). If it is up and not
rebooted, then its probability of remaining up decreases by \(\rho\) for each
failed neighbour:
\[
\mathbb P(X'_i=1\mid x,a)
=
\left[
p_{\mathrm{base}}
-
\rho\sum_{j\in\mathcal N(i)}(1-x_j)
\right]_{[0,1]},
\]
where \([\cdot]_{[0,1]}\) denotes clipping to \([0,1]\). Thus the tabular
problem has \(2^N\) states, \(N+1\) actions, and \((N+1)2^N\) state-action
pairs, while the transition law has local ring structure.

\paragraph{Scaling sweep.}
The scaling experiment in \autoref{fig:sysadmin} runs for systems of sizes \(N=2,\ldots,7\). 
For each \(N\), the
full-resolution reference return is the exact value of the final concrete
baseline policy under the same \(100\)-sweep full-resolution compute budget. The
selected checkpoint is the first adaptive checkpoint whose grounded concrete
policy return is at least \(99\%\) of this reference. The star in
\autoref{fig:sysadmin} is then placed at the rate-distortion row whose
\(\beta\) matches that selected checkpoint. Distortion in the scaling plots is
normalised by the maximum entry of the fixed-point state-action metric for that
\(N\), and effective size is reported using the information-equivalent counts
described in \autoref{app:shared-experiment-protocol}.

\paragraph{Factor probes.}
For the factor analysis in \autoref{tab:sysadmin-factor-representatives}, we
load the adaptive encoder at the selected near-optimal checkpoint for each
\(N=2,\ldots,7\). The flat decoder representatives are projected to labels
\(\bar S\) and \(\bar A\) as above. For a discrete state factor \(F(S)\), we
compute \(I(F;\bar S)\). For a discrete action-side factor \(F(S,A)\), we compute
\(I(F;\bar A\mid \bar S)\). Relevance is the same quantity divided by the
entropy of the probed factor, so it lies in \([0,1]\) when the factor has
positive entropy. State probes include global load variables, adjacent-pair and
triplet histograms around the ring, failure-run statistics, isolated-failure
indicators, and counts of vulnerable or critical alive machines. Action probes
include whether the action is \(\mathrm{noop}\) or a reboot, target identity,
target up/down status, wasted reboot indicators, failed-neighbour counts,
target-centred triplet and five-site patterns, and the size of the alive or
failed run containing the target. To avoid reporting many redundant probes, the
table first groups related factors, chooses the strongest raw-mutual-information
factor within each group, and then reports the top four state-side and action-side
representatives averaged across \(N=2,\ldots,7\).

The factors reported in \autoref{tab:sysadmin-factor-representatives} are
defined as follows.
\begin{itemize}
    \item \emph{Triplet hist.} is the
eight-dimensional histogram of cyclic local patterns
\((x_{i-1},x_i,x_{i+1})\), one count for each binary triplet.
    \item \emph{Pair hist.}
    is the four-dimensional histogram of cyclic adjacent patterns
\((x_i,x_{i+1})\), with one count for each binary pair.
    \item \emph{Failed count} is
\(\sum_i(1-x_i)\), the number of down machines.
    \item \emph{Max fail run} is the
length of the longest contiguous cyclic run of down machines. On the action
side, the target of \(\mathrm{reboot}_i\) is machine \(i\), while
\(\mathrm{noop}\) is assigned a separate no-target symbol.
    \item \emph{Target 5-bit}
is the five-site neighbourhood
\((x_{i-2},x_{i-1},x_i,x_{i+1},x_{i+2})\) around the reboot target, with cyclic
indexing.
    \item \emph{Target run size} is the size of the contiguous failed run
containing the reboot target, and is zero when the target machine is up.
    \item \emph{Target index} is the identity of the machine selected for reboot.
    \item \emph{isolated target} indicates whether the reboot target is down while both
of its neighbours are up.
\end{itemize}


\begin{thebibliography}{70}
\providecommand{\natexlab}[1]{#1}
\providecommand{\doi}[1]{doi: #1}

\bibitem[Konidaris(2019)]{konidaris2019necessity}
George Konidaris.
\newblock On the necessity of abstraction.
\newblock \emph{Current opinion in behavioral sciences}, 29:\penalty0 1--7,
  2019.

\bibitem[Ho et~al.(2019)Ho, Abel, Griffiths, and Littman]{ho2019value}
Mark~K Ho, David Abel, Thomas~L Griffiths, and Michael~L Littman.
\newblock The value of abstraction.
\newblock \emph{Current opinion in behavioral sciences}, 29:\penalty0 111--116,
  2019.

\bibitem[Abel(2019)]{abel2019theory}
David Abel.
\newblock A theory of state abstraction for reinforcement learning.
\newblock \emph{Proceedings of the AAAI Conference on Artificial Intelligence},
  33\penalty0 (01):\penalty0 9876--9877, 2019.
\newblock \doi{10.1609/aaai.v33i01.33019876}.

\bibitem[Allen(2023)]{allen2023structured}
Cameron S. Allen.
\newblock \emph{Structured Abstractions for General-Purpose Decision Making}.
\newblock PhD thesis, Brown University, 2023.

\bibitem[Ravindran and Barto(2003)]{ravindran2004algebraic}
Balaraman Ravindran and Andrew~G. Barto.
\newblock An algebraic approach to abstraction in reinforcement learning.
\newblock Technical report, University of Massachusetts Amherst, 2003.

\bibitem[van~der Pol et~al.(2020)van~der Pol, Worrall, van Hoof, Oliehoek, and
  Welling]{van2020mdp}
Elise van~der Pol, Daniel~E. Worrall, Herke van Hoof, Frans~A. Oliehoek, and Max
  Welling.
\newblock MDP homomorphic networks: Group symmetries in reinforcement learning.
\newblock \emph{Advances in Neural Information Processing Systems},
  33:\penalty0 4199--4210, 2020.

\bibitem[Givan et~al.(2003)Givan, Dean, and Greig]{givan2003equivalence}
Robert Givan, Thomas Dean, and Matthew Greig.
\newblock Equivalence notions and model minimization in {M}arkov decision
  processes.
\newblock \emph{Artificial Intelligence}, 147\penalty0 (1--2):\penalty0
  163--223, 2003.

\bibitem[Li et~al.(2006)Li, Walsh, and Littman]{li2006towards}
Lihong Li, Thomas~J. Walsh, and Michael~L. Littman.
\newblock Towards a unified theory of state abstraction for {MDP}s.
\newblock In \emph{Proceedings of the Ninth International Symposium on
  Artificial Intelligence and Mathematics}, 2006.

\bibitem[Ferns et~al.(2004)Ferns, Panangaden, and Precup]{ferns2004metrics}
Norm Ferns, Prakash Panangaden, and Doina Precup.
\newblock Metrics for finite {M}arkov decision processes.
\newblock In \emph{Proceedings of the 20th Conference on Uncertainty in
  Artificial Intelligence}, pages 162--169, 2004.

\bibitem[Ferns et~al.(2011)Ferns, Panangaden, and
  Precup]{ferns2011bisimulation}
Norm Ferns, Prakash Panangaden, and Doina Precup.
\newblock Bisimulation metrics for continuous {M}arkov decision processes.
\newblock \emph{SIAM Journal on Computing}, 40\penalty0 (6):\penalty0
  1662--1714, 2011.

\bibitem[Taylor et~al.(2008)Taylor, Precup, and Panangaden]{taylor2008bounding}
Jonathan~J. Taylor, Doina Precup, and Prakash Panangaden.
\newblock Bounding performance loss in approximate {MDP} homomorphisms.
\newblock In \emph{Advances in Neural Information Processing Systems}, 2008.

\bibitem[Abel et~al.(2016)Abel, Hershkowitz, and Littman]{abel2016near}
David Abel, David Hershkowitz, and Michael~L. Littman.
\newblock Near optimal behavior via approximate state abstraction.
\newblock In \emph{Proceedings of The 33rd International Conference on Machine
  Learning}, volume~48 of \emph{Proceedings of Machine Learning Research},
  pages 2915--2923, 2016.

\bibitem[Shannon(1959)]{shannon1959coding}
Claude~E. Shannon.
\newblock Coding theorems for a discrete source with a fidelity criterion.
\newblock \emph{IRE National Convention Record}, 4:\penalty0 142--163, 1959.

\bibitem[Tishby et~al.(1999)Tishby, Pereira, and Bialek]{tishby2000information}
Naftali Tishby, Fernando~C. Pereira, and William Bialek.
\newblock The information bottleneck method.
\newblock In \emph{The 37th annual Allerton Conference on Communication,
  Control, and Computing}, pages 368--377, 1999.

\bibitem[Ortner(2013)]{ortner2013adaptive}
Ronald Ortner.
\newblock Adaptive aggregation for reinforcement learning in average reward
  {M}arkov decision processes.
\newblock \emph{Annals of Operations Research}, 208:\penalty0 321--336, 2013.

\bibitem[Jiang(2018)]{jiang2018notes}
Nan Jiang.
\newblock Notes on state abstractions.
\newblock Lecture notes, University of Illinois at Urbana-Champaign, 2018.

\bibitem[Clarke et~al.(2003)Clarke, Grumberg, Jha, Lu, and
  Veith]{clarke2000cegar}
Edmund~M. Clarke, Orna Grumberg, Somesh Jha, Yuan Lu, and Helmut Veith.
\newblock Counterexample-guided abstraction refinement for symbolic model
  checking.
\newblock \emph{Journal of the ACM}, 50\penalty0 (5):\penalty0 752--794, 2003.

\bibitem[Abate et~al.(2024)Abate, Giacobbe, and Schnitzer]{abate2024bisimulation}
Alessandro Abate, Mirco Giacobbe, and Yannik Schnitzer.
\newblock Bisimulation learning, 2024.
\newblock arXiv:2405.15723.

\bibitem[Coppola et~al.(2025)Coppola, Schnitzer, Giacobbe, Abate, and
  Mazo~Jr]{coppola2025existence}
Rudi Coppola, Yannik Schnitzer, Mirco Giacobbe, Alessandro Abate, and Manuel
  Mazo~Jr.
\newblock Existence and synthesis of multi-resolution approximate bisimulations
  for continuous-state dynamical systems.
\newblock \emph{arXiv preprint arXiv:2509.17739}, 2025.

\bibitem[Guestrin et~al.(2003)Guestrin, Koller, Parr, and
  Venkataraman]{guestrin2003efficient}
Carlos Guestrin, Daphne Koller, Ronald Parr, and Shobha Venkataraman.
\newblock Efficient solution algorithms for factored MDPs.
\newblock \emph{Journal of Artificial Intelligence Research}, 19:\penalty0
  399--468, 2003.
\newblock \doi{10.1613/jair.1000}.

\bibitem[Abel et~al.(2020)Abel, Umbanhowar, Khetarpal, Arumugam, Precup, and
  Littman]{abel2020value}
David Abel, Nathan Umbanhowar, Khimya Khetarpal, Dilip Arumugam, Doina Precup,
  and Michael~L. Littman.
\newblock Value preserving state-action abstractions.
\newblock In \emph{Proceedings of the 23rd International Conference on
  Artificial Intelligence and Statistics}, volume 108, pages 1639--1650, 2020.

\bibitem[Jiang et~al.(2015)Jiang, Kulesza, and Singh]{jiang2015abstraction}
Nan Jiang, Alex Kulesza, and Satinder Singh.
\newblock Abstraction selection in model-based reinforcement learning.
\newblock In \emph{Proceedings of the 32nd International Conference on Machine
  Learning}, volume~37 of \emph{Proceedings of Machine Learning Research},
  pages 179--188, 2015.

\bibitem[Ferns and Precup(2014)]{ferns2014bisimulation}
Norm Ferns and Doina Precup.
\newblock Bisimulation metrics are optimal value functions.
\newblock In \emph{Proceedings of the 30th Conference on Uncertainty in
  Artificial Intelligence}, pages 210--219, 2014.

\bibitem[Rezaei-Shoshtari et~al.(2022)Rezaei-Shoshtari, Zhao, Panangaden,
  Meger, and Precup]{zhao2022continuous}
Sahand Rezaei-Shoshtari, Rosie Zhao, Prakash Panangaden, David Meger, and Doina
  Precup.
\newblock Continuous {MDP} homomorphisms and homomorphic policy gradient.
\newblock In \emph{Advances in Neural Information Processing Systems}, 2022.

\bibitem[Castro(2020)]{castro2020scalable}
Pablo~Samuel Castro.
\newblock Scalable methods for computing state similarity in deterministic
  {M}arkov decision processes.
\newblock \emph{Proceedings of the AAAI Conference on Artificial Intelligence},
  34\penalty0 (06):\penalty0 10069--10076, 2020.
\newblock \doi{10.1609/aaai.v34i06.6564}.

\bibitem[Panangaden et~al.(2024)Panangaden, Rezaei-Shoshtari, Zhao, Meger, and
  Precup]{panangaden2024policy}
Prakash Panangaden, Sahand Rezaei-Shoshtari, Rosie Zhao, David Meger, and Doina
  Precup.
\newblock Policy gradient methods in the presence of symmetries and state
  abstractions.
\newblock \emph{Journal of Machine Learning Research}, 25:\penalty0 1--57,
  2024.

\bibitem[Gelada et~al.(2019)Gelada, Kumar, Buckman, Nachum, and
  Bellemare]{gelada2019deepmdp}
Carles Gelada, Saurabh Kumar, Jacob Buckman, Ofir Nachum, and Marc~G.
  Bellemare.
\newblock {DeepMDP}: Learning continuous latent space models for representation
  learning.
\newblock In \emph{Proceedings of the 36th International Conference on Machine
  Learning}, volume~97 of \emph{Proceedings of Machine Learning Research},
  pages 2170--2179, 2019.

\bibitem[Zhang et~al.(2021)Zhang, McAllister, Calandra, Gal, and
  Levine]{zhang2021learning}
Amy Zhang, Rowan McAllister, Roberto Calandra, Yarin Gal, and Sergey Levine.
\newblock Learning invariant representations for reinforcement learning without
  reconstruction.
\newblock In \emph{International Conference on Learning Representations}, 2021.

\bibitem[Castro et~al.(2021)Castro, Kastner, Panangaden, and
  Rowland]{castro2021mico}
Pablo~Samuel Castro, Tyler Kastner, Prakash Panangaden, and Mark Rowland.
\newblock {MICo}: Improved representations via sampling-based state similarity
  for {M}arkov decision processes.
\newblock In \emph{Advances in Neural Information Processing Systems}, 2021.

\bibitem[Kemertas and Aumentado-Armstrong(2021)]{kemertas2021towards}
Mete Kemertas and Tristan Aumentado-Armstrong.
\newblock Towards robust bisimulation metric learning.
\newblock In \emph{Advances in Neural Information Processing Systems}, 2021.

\bibitem[Agarwal et~al.(2021)Agarwal, Machado, Castro, and
  Bellemare]{agarwal2021contrastive}
Rishabh Agarwal, Marlos~C. Machado, Pablo~Samuel Castro, and Marc~G. Bellemare.
\newblock Contrastive behavioral similarity embeddings for generalization in
  reinforcement learning.
\newblock In \emph{International Conference on Learning Representations}, 2021.

\bibitem[Abel et~al.(2019)Abel, Arumugam, Asadi, Jinnai, Littman, and
  Wong]{abel2019state}
David Abel, Dilip Arumugam, Kavosh Asadi, Yuu Jinnai, Michael~L. Littman, and
  Lawson L.~S. Wong.
\newblock State abstraction as compression in apprenticeship learning.
\newblock \emph{Proceedings of the AAAI Conference on Artificial Intelligence},
  33\penalty0 (01):\penalty0 3134--3142, 2019.
\newblock \doi{10.1609/aaai.v33i01.33013134}.

\bibitem[Biza et~al.(2021)Biza, Platt, van~de Meent, and
  Wong]{biza2021learning}
Ondrej Biza, Robert Platt, Jan-Willem van~de Meent, and Lawson L.~S. Wong.
\newblock Learning discrete state abstractions with deep variational inference.
\newblock In \emph{Third Symposium on Advances in Approximate Bayesian
  Inference}, 2021.
\bibitem[Delgrange et~al.(2022)Delgrange, Now{\'e}, and
  P{\'e}rez]{delgrange2022distillation}
Florent Delgrange, Ann Now{\'e}, and Guillermo~A. P{\'e}rez.
\newblock Distillation of {RL} policies with formal guarantees via variational
  abstraction of {M}arkov decision processes.
\newblock \emph{Proceedings of the AAAI Conference on Artificial Intelligence},
  36\penalty0 (6):\penalty0 6497--6505, 2022.
\newblock \doi{10.1609/aaai.v36i6.20602}.

\bibitem[Zhu et~al.(2022)Zhu, Huang, Zhang, and Zhu]{xu2022wdibs}
Xianchao Zhu, Tianyi Huang, Ruiyuan Zhang, and William Zhu.
\newblock {WDIBS}: Wasserstein deterministic information bottleneck for state
  abstraction to balance state-compression and performance.
\newblock \emph{Applied Intelligence}, 52\penalty0 (6):\penalty0 6316--6329,
  2022.
\newblock \doi{10.1007/s10489-021-02787-4}.

\bibitem[Goyal et~al.(2019)Goyal, Islam, Strouse, Ahmed, Botvinick, Larochelle,
Bengio, and Levine]{goyal2019infobot}
Anirudh Goyal, Riashat Islam, Daniel Strouse, Zafarali Ahmed, Matthew
  Botvinick, Hugo Larochelle, Yoshua Bengio, and Sergey Levine.
\newblock {InfoBot}: Transfer and exploration via the information bottleneck.
\newblock In \emph{International Conference on Learning Representations}, 2019.

\bibitem[Igl et~al.(2019)Igl, Ciosek, Li, Tschiatschek, Zhang, Devlin, and
  Hofmann]{igl2019generalization}
Maximilian Igl, Kamil Ciosek, Yingzhen Li, Sebastian Tschiatschek, Cheng Zhang,
  Sam Devlin, and Katja Hofmann.
\newblock Generalization in reinforcement learning with selective noise
  injection and information bottleneck.
\newblock In \emph{Advances in Neural Information Processing Systems}, 2019.

\bibitem[Clauw et~al.(2025)Clauw, Polani, and
  Catenacci~Volpi]{clauw2025ibtransfer}
Kenzo Clauw, Daniel Polani, and Nicola Catenacci~Volpi.
\newblock A theoretical analysis of information bottlenecks for zero-shot
  transfer in reinforcement learning, 2025.
\newblock OpenReview preprint, ARLET 2025.

\bibitem[Freed et~al.(2025)Freed, Calandra, Schneider, and
  Choset]{freed2025vibes}
Benjamin Freed, Roberto Calandra, Jeff Schneider, and Howie Choset.
\newblock Distractor-robust reinforcement learning via variational
  bisimulation, 2025.
\newblock OpenReview preprint, submitted to ICLR 2026.

\bibitem[Arumugam and Van~Roy(2021{\natexlab{a}})]{arumugam2021deciding}
Dilip Arumugam and Benjamin Van~Roy.
\newblock Deciding what to learn: A rate-distortion approach.
\newblock In \emph{Proceedings of the 38th International Conference on Machine
  Learning}, volume 139 of \emph{Proceedings of Machine Learning Research},
  pages 373--382, 2021{\natexlab{a}}.

\bibitem[Arumugam and Van~Roy(2021{\natexlab{b}})]{arumugam2021value}
Dilip Arumugam and Benjamin Van~Roy.
\newblock The value of information when deciding what to learn.
\newblock In \emph{Advances in Neural Information Processing Systems},
  volume~34, pages 9816--9827, 2021{\natexlab{b}}.

\bibitem[Arumugam and Van~Roy(2022{\natexlab{a}})]{arumugam2022value}
Dilip Arumugam and Benjamin Van~Roy.
\newblock Deciding what to model: Value-equivalent sampling for reinforcement
  learning.
\newblock In \emph{Advances in Neural Information Processing Systems},
  2022{\natexlab{a}}.

\bibitem[Arumugam and Van~Roy(2022{\natexlab{b}})]{arumugam2022between}
Dilip Arumugam and Benjamin Van~Roy.
\newblock Between rate-distortion theory \& value equivalence in model-based
  reinforcement learning, 2022{\natexlab{b}}.
\newblock Accepted to the Multi-Disciplinary Conference on Reinforcement
  Learning and Decision Making (RLDM) 2022.

\bibitem[Arumugam et~al.(2022)Arumugam, Ho, Goodman, and
  Van~Roy]{arumugam2022rate}
Dilip Arumugam, Mark~K. Ho, Noah~D. Goodman, and Benjamin Van~Roy.
\newblock On rate-distortion theory in capacity-limited cognition \&
  reinforcement learning, 2022.
\newblock NeurIPS 2022 Workshop on Information-Theoretic Principles in
  Cognitive Systems.

\bibitem[Polani(2009)]{polani2009information}
Daniel Polani.
\newblock Information: Currency of life?
\newblock \emph{HFSP Journal}, 3\penalty0 (5):\penalty0 307--316, 2009.

\bibitem[Tishby and Polani(2011)]{tishby2011information}
Naftali Tishby and Daniel Polani.
\newblock Information theory of decisions and actions.
\newblock In \emph{Perception-Action Cycle}, pages 601--636. Springer, 2011.

\bibitem[Rubin et~al.(2012)Rubin, Shamir, and Tishby]{rubin2012trading}
Jonathan Rubin, Ohad Shamir, and Naftali Tishby.
\newblock Trading value and information in {MDP}s.
\newblock In \emph{Decision Making with Imperfect Decision Makers}, pages
  57--74. Springer, 2012.

\bibitem[Dean et~al.(1997)Dean, Givan, and Leach]{dean1997model}
Thomas Dean, Robert Givan, and Sonia Leach.
\newblock Model reduction techniques for computing approximately optimal
  solutions for {M}arkov decision processes.
\newblock In \emph{Proceedings of the 13th Conference on Uncertainty in
  Artificial Intelligence}, pages 124--131, 1997.

\bibitem[Sutton and Barto(2018)]{SuttonBarto2018}
Richard~S. Sutton and Andrew~G. Barto.
\newblock \emph{Reinforcement Learning: An Introduction}.
\newblock MIT Press, 2nd edition, 2018.

\bibitem[Taylor(2008)]{taylor2008lax}
Jonathan Taylor.
\newblock \emph{Lax probabilistic bisimulation}.
\newblock PhD thesis, McGill University, 2008.

\bibitem[Blahut(1972)]{blahut1972computation}
Richard~E. Blahut.
\newblock Computation of channel capacity and rate-distortion functions.
\newblock \emph{IEEE Transactions on Information Theory}, 18\penalty0
  (4):\penalty0 460--473, 1972.

\bibitem[Arimoto(1972)]{arimoto1972algorithm}
Suguru Arimoto.
\newblock An algorithm for computing the capacity of arbitrary discrete
  memoryless channels.
\newblock \emph{IEEE Transactions on Information Theory}, 18\penalty0
  (1):\penalty0 14--20, 1972.

\bibitem[Slonim and Tishby(1999)]{slonim1999agglomerative}
Noam Slonim and Naftali Tishby.
\newblock Agglomerative information bottleneck.
\newblock \emph{Advances in neural information processing systems}, 12, 1999.

\bibitem[Sutton et~al.(1999)Sutton, Precup, and Singh]{sutton1999between}
Richard~S. Sutton, Doina Precup, and Satinder Singh.
\newblock Between {MDP}s and semi-{MDP}s: A framework for temporal abstraction
  in reinforcement learning.
\newblock \emph{Artificial Intelligence}, 112\penalty0 (1--2):\penalty0
  181--211, 1999.

\bibitem[Dietterich(2000)]{dietterich2000hierarchical}
Thomas~G Dietterich.
\newblock Hierarchical reinforcement learning with the MAXQ value function
  decomposition.
\newblock \emph{Journal of artificial intelligence research}, 13:\penalty0
  227--303, 2000.

\bibitem[Chevalier-Boisvert et~al.(2023)Chevalier-Boisvert, Dai, Towers,
  de Lazcano, Willems, Lahlou, Pal, Castro, and
  Terry]{chevalier2023minigrid}
Maxime Chevalier-Boisvert, Bolun Dai, Mark Towers, Rodrigo de Lazcano, Lucas
  Willems, Salem Lahlou, Suman Pal, Pablo~Samuel Castro, and Jordan Terry.
\newblock Minigrid \& Miniworld: Modular \& customizable reinforcement learning
  environments for goal-oriented tasks.
\newblock \emph{Advances in Neural Information Processing Systems},
  36:\penalty0 73383--73394, 2023.

\bibitem[Bacon et~al.(2017)Bacon, Harb, and Precup]{bacon2017optioncritic}
Pierre-Luc Bacon, Jean Harb, and Doina Precup.
\newblock The option-critic architecture.
\newblock In \emph{Proceedings of the 31st AAAI Conference on Artificial
  Intelligence}, 2017.

\bibitem[Valle-Perez et~al.(2019)Valle-Perez, Camargo, and
  Louis]{valle2018deep}
Guillermo Valle-P{\'e}rez, Chico~Q Camargo, and Ard~A Louis.
\newblock Deep learning generalizes because the parameter-function map is
  biased towards simple functions.
\newblock In \emph{International Conference on Learning Representations}, 2019.
\newblock arXiv:1805.08522.

\bibitem[Huh et~al.(2021)Huh, Mobahi, Zhang, Cheung, Agrawal, and
  Isola]{huh2021low}
Minyoung Huh, Hossein Mobahi, Richard Zhang, Brian Cheung, Pulkit Agrawal, and
  Phillip Isola.
\newblock The low-rank simplicity bias in deep networks, 2021.
\bibitem[Shai et~al.(2026)Shai, Amdahl-Culleton, Christensen, Bigelow, Rosas,
  Boyd, Alt, Ray, and Riechers]{shai2026transformers}
Adam Shai, Loren Amdahl-Culleton, Casper~L. Christensen, Henry~R. Bigelow,
  Fernando~E. Rosas, Alexander~B. Boyd, Eric~A. Alt, Kyle~J. Ray, and Paul~M.
  Riechers.
\newblock Transformers learn factored representations, 2026.
\bibitem[Rosch et~al.(1976)Rosch, Mervis, Gray, Johnson, and
  Boyes-Braem]{rosch1976basic}
Eleanor Rosch, Carolyn~B Mervis, Wayne~D Gray, David~M Johnson, and Penny
  Boyes-Braem.
\newblock Basic objects in natural categories.
\newblock \emph{Cognitive psychology}, 8\penalty0 (3):\penalty0 382--439, 1976.

\bibitem[Newport(1990)]{newport1990maturational}
Elissa~L Newport.
\newblock Maturational constraints on language learning.
\newblock \emph{Cognitive science}, 14\penalty0 (1):\penalty0 11--28, 1990.

\bibitem[Adolph(2008)]{adolph2008learning}
Karen~E Adolph.
\newblock Learning to move.
\newblock \emph{Current directions in psychological science}, 17\penalty0
  (3):\penalty0 213--218, 2008.

\bibitem[Kemeny and Snell(1960)]{kemeny1960finite}
John~G. Kemeny and J.~Laurie Snell.
\newblock \emph{Finite Markov Chains}.
\newblock D. Van Nostrand, Princeton, NJ, 1960.

\bibitem[Zang et~al.(2023)Zang, Li, Zhang, Liu, Sun, Islam, Tachet~des Combes,
  and Laroche]{liao2023pitfalls}
Hongyu Zang, Xin Li, Leiji Zhang, Yang Liu, Baigui Sun, Riashat Islam, R{\'e}mi
  Tachet~des Combes, and Romain Laroche.
\newblock Understanding and addressing the pitfalls of bisimulation-based
  representations in offline reinforcement learning.
\newblock In \emph{Advances in Neural Information Processing Systems},
  volume~36, 2023.

\bibitem[Kaufman and Rousseeuw(1990)]{kaufman1990finding}
Leonard Kaufman and Peter~J. Rousseeuw.
\newblock \emph{Finding Groups in Data: An Introduction to Cluster Analysis}.
\newblock Wiley, 1990.

\bibitem[Lloyd(1982)]{lloyd1982least}
Stuart~P. Lloyd.
\newblock Least squares quantization in {PCM}.
\newblock \emph{IEEE Transactions on Information Theory}, 28\penalty0
  (2):\penalty0 129--137, 1982.

\bibitem[Rose(1998)]{rose1998deterministic}
Kenneth Rose.
\newblock Deterministic annealing for clustering, compression, classification,
  regression, and related optimization problems.
\newblock \emph{Proceedings of the IEEE}, 86\penalty0 (11):\penalty0
  2210--2239, 1998.

\bibitem[Csisz{\'a}r and Tusn{\'a}dy(1984)]{csiszar1984information}
Imre Csisz{\'a}r and G{\'a}bor Tusn{\'a}dy.
\newblock Information geometry and alternating minimization procedures.
\newblock \emph{Statistics and Decisions, Supplement Issue}, 1:\penalty0
  205--237, 1984.

\end{thebibliography}
\end{document}